\documentclass{article}

     \PassOptionsToPackage{numbers, sort}{natbib}
\usepackage[final]{neurips_2021}

\usepackage[utf8]{inputenc} % allow utf-8 input
\usepackage[T1]{fontenc}    % use 8-bit T1 fonts
\usepackage{hyperref}       % hyperlinks
\usepackage{url}            % simple URL typesetting
\usepackage{booktabs}       % professional-quality tables
\usepackage{amsfonts}       % blackboard math symbols
\usepackage{nicefrac}       % compact symbols for 1/2, etc.
\usepackage{microtype}      % microtypography
\usepackage{xcolor}         % colors

\usepackage{amsmath,amssymb,amsthm}
\usepackage[ruled,noline]{algorithm2e}
\usepackage[pdftex]{graphicx}
\usepackage{array}

\newcolumntype{C}[1]{>{\centering\arraybackslash}b{\dimexpr #1\linewidth}}
\newcolumntype{D}[1]{>{\centering\arraybackslash}b{\dimexpr #1\linewidth+2\tabcolsep}}
\newcolumntype{E}[1]{>{\centering\arraybackslash}b{\dimexpr #1\linewidth+3\tabcolsep}}
\newcolumntype{L}[1]{>{\raggedright\arraybackslash}b{\dimexpr #1\linewidth}}
\newcolumntype{R}[1]{>{\raggedleft\arraybackslash}b{\dimexpr #1\linewidth}}

\newtheorem{theorem}{Theorem}
\newtheorem{lemma}{Lemma}
\newtheorem{definition}{Definition}

\title{Generalized Proximal Policy Optimization\\with Sample Reuse}

\author{%
   James Queeney \\
   Division of Systems Engineering \\
   Boston University \\
   \texttt{jqueeney@bu.edu} \\
   \And
   Ioannis Ch.~Paschalidis \\
   Department of Electrical and Computer Engineering \\
   Division of Systems Engineering \\
   Boston University \\
   \texttt{yannisp@bu.edu} \\
   \And
   Christos G.~Cassandras \\
   Department of Electrical and Computer Engineering \\
   Division of Systems Engineering \\
   Boston University \\
   \texttt{cgc@bu.edu} \\
}

\begin{document}

\maketitle

\begin{abstract}
In real-world decision making tasks, it is critical for data-driven reinforcement learning methods to be both stable and sample efficient. On-policy methods typically generate reliable policy improvement throughout training, while off-policy methods make more efficient use of data through sample reuse. In this work, we combine the theoretically supported stability benefits of on-policy algorithms with the sample efficiency of off-policy algorithms. We develop policy improvement guarantees that are suitable for the off-policy setting, and connect these bounds to the clipping mechanism used in Proximal Policy Optimization. This motivates an off-policy version of the popular algorithm that we call Generalized Proximal Policy Optimization with Sample Reuse. We demonstrate both theoretically and empirically that our algorithm delivers improved performance by effectively balancing the competing goals of stability and sample efficiency.
\end{abstract}

%%%%%%%%%%%%%%%%%%%%%%%%%%%%%%%%%%%%%%%%%%%%%%%%%%%%%%%%%%%%

\section{Introduction}

In recent years, model-free deep reinforcement learning has been used to successfully solve complex simulated control tasks \citep{duan_2016}. Unfortunately, real-world adoption of these techniques remains limited. High-stakes real-world decision making settings demand methods that deliver stable, reliable performance throughout training. In addition, real-world data collection can be difficult and expensive, so learning must make efficient use of limited data. The combination of these requirements is not an easy task, as stability and sample efficiency often represent competing interests. Existing model-free deep reinforcement learning algorithms often focus on one of these goals, and as a result sacrifice performance with respect to the other.

On-policy reinforcement learning methods such as Proximal Policy Optimization (PPO) \citep{schulman_2017} deliver stable performance throughout training due to their connection to theoretical policy improvement guarantees. These methods are motivated by a lower bound on the expected performance loss at every update, which can be approximated using samples generated by the current policy. The theoretically supported stability of these methods is very attractive, but the need for on-policy data and the high-variance nature of reinforcement learning often requires significant data to be collected between every update, resulting in high sample complexity and slow learning.

Off-policy algorithms address the issue of high sample complexity by storing samples in a replay buffer, which allows data to be reused to calculate multiple policy updates. The ability to reuse samples improves learning speed, but also causes the distribution of data to shift away from the distribution generated by the current policy. This distribution shift invalidates the standard performance guarantees used in on-policy methods, and can lead to instability in the training process. Popular off-policy algorithms often require various implementation tricks and extensive hyperparameter tuning to control the instability caused by off-policy data.

By combining the attractive features of on-policy and off-policy methods in a principled way, we can balance the competing goals of stability and sample efficiency required in real-world decision making. We consider the popular on-policy algorithm PPO as our starting point due to its theoretically supported stable performance, and develop an off-policy variant with principled sample reuse that we call \emph{Generalized Proximal Policy Optimization with Sample Reuse (GePPO)}. Our algorithm is based on the following main contributions:
\begin{enumerate}
\item We extend existing policy improvement guarantees to the off-policy setting, resulting in a lower bound that can be approximated using data from all recent policies.
\item We develop connections between the clipping mechanism used in PPO and the penalty term in our policy improvement lower bound, which motivates a generalized clipping mechanism for off-policy data.
\item We propose an adaptive learning rate method based on the same penalty term that more closely connects theory and practice.
\end{enumerate}

We provide theoretical evidence that our algorithm effectively balances the goals of stability and sample efficiency, and we demonstrate the strong performance of our approach through experiments on high-dimensional continuous control tasks in OpenAI Gym's MuJoCo environments \citep{brockman_2016,todorov_2012}. 

%%%%%%%%%%%%%%%%%%%%%%%%%%%%%%%%%%%%%%%%%%%%%%%%%%%%%%%%%%%%

\section{Related work}

\paragraph{On-policy policy improvement methods}

The goal of monotonic policy improvement was first introduced by \citet{kakade_2002} in Conservative Policy Iteration, which maximizes a lower bound on policy improvement that can be constructed using samples from the current policy. This theory of policy improvement has served as a fundamental building block in the design of on-policy deep reinforcement learning methods, including the popular algorithms Trust Region Policy Optimization (TRPO) \citep{schulman_2015} and Proximal Policy Optimization (PPO) \citep{schulman_2017}. TRPO achieves approximate policy improvement by enforcing a Kullback-Leibler (KL) divergence trust region, while PPO does so by clipping the probability ratio between current and future policies.

Due to the strong performance of TRPO and PPO, there has been substantial interest in better understanding these methods. \citet{engstrom_2020} and \citet{andrychowicz_2021} both performed extensive empirical analysis on the various implementation choices in these algorithms, while other research has focused on the clipping mechanism used in PPO. \citet{wang_2019} and \citet{wang_2020} both proposed modifications to the clipping mechanism based on a KL divergence trust region. To the best of our knowledge, we are the first to directly relate the clipping mechanism in PPO to the total variation distance between policies. \citet{wang_2020} also proposed a rollback operation to keep probability ratios close to the clipping range. We accomplish a similar goal by considering an adaptive learning rate.

\paragraph{Sample efficiency with off-policy data}

A common approach to improving the sample efficiency of model-free reinforcement learning is to reuse samples collected under prior policies. Popular off-policy policy gradient approaches such as Deep Deterministic Policy Gradient (DDPG) \citep{lillicrap_2016}, Twin Delayed DDPG (TD3) \citep{fujimoto_2018}, and Soft Actor-Critic (SAC) \citep{haarnoja_2018} accomplish this by storing data in a replay buffer and sampling from this buffer to calculate policy updates. Note that these methods are not motivated by policy improvement guarantees, and do not directly control the bias introduced by off-policy data.

Other approaches have combined on-policy and off-policy policy gradients, with the goal of balancing the variance of on-policy methods and the bias of off-policy methods \citep{gu_2017_qprop, gu_2017_ipg, odonoghue_2017, wang_2017, fakoor_2020}. \citet{gu_2017_ipg} demonstrated that the bias introduced by off-policy data is related to the KL divergence between the current policy and the behavior policy used to generate the data. \citet{fakoor_2020} considered a related KL divergence as a penalty term in their objective, while \citet{wang_2017} approximately controlled this KL divergence by applying a trust region around a target policy. These methods are related to the penalty term that appears in our generalized policy improvement lower bound, which can be bounded by a penalty that depends on KL divergence.

Finally, there have been heuristic attempts to incorporate off-policy data into PPO \citep{sovrano_2019,holubar_2020}. However, unlike our approach, these methods do not account for the distribution shift caused by off-policy data that invalidates the theoretical support for PPO.

%%%%%%%%%%%%%%%%%%%%%%%%%%%%%%%%%%%%%%%%%%%%%%%%%%%%%%%%%%%%

\section{Preliminaries}

\paragraph{Reinforcement learning framework}

We consider an infinite-horizon, discounted Markov Decision Process (MDP) defined by the tuple $(\mathcal{S},\mathcal{A},p,r,\rho_0,\gamma)$, where $\mathcal{S}$ is the set of states, $\mathcal{A}$ is the set of actions, $p: \mathcal{S} \times \mathcal{A} \rightarrow \Delta_{\mathcal{S}}$ is the transition probability function, $r: \mathcal{S} \times \mathcal{A} \rightarrow \mathbb{R}$ is the reward function, $\rho_0$ is the initial state distribution, and $\gamma$ is the discount rate.

We model the agent's decisions as a stationary policy $\pi: \mathcal{S} \rightarrow \Delta_{\mathcal{A}}$. Our goal is to choose a policy that maximizes the expected total discounted rewards $J(\pi) = \mathop{\mathbb{E}}_{\tau \sim \pi} \left[ \sum_{t=0}^{\infty} \gamma^t r(s_t,a_t) \right]$, where $\tau \sim \pi$ represents a trajectory sampled according to $s_0 \sim \rho_0$, $a_t \sim \pi(\, \cdot \mid s_t)$, and $s_{t+1} \sim p(\, \cdot \mid s_t,a_t)$. A policy $\pi$ also induces a normalized discounted state visitation distribution $d^{\pi}$, where $d^{\pi}(s) = (1-\gamma) \sum_{t=0}^{\infty} \gamma^t \mathbb{P}(s_t = s \mid \rho_0, \pi, p)$. We write the corresponding normalized discounted state-action visitation distribution as $d^{\pi}(s,a) = d^{\pi}(s) \pi(a \mid s)$, where we make it clear from the context whether $d^{\pi}$ refers to a distribution over states or state-action pairs.

We denote the state value function of $\pi$ as $V^{\pi}(s) = \mathop{\mathbb{E}}_{\tau \sim \pi} \left[ \sum_{t=0}^{\infty} \gamma^t r(s_t,a_t) \mid s_0 = s \right]$, the state-action value function as $Q^{\pi}(s,a) = \mathop{\mathbb{E}}_{\tau \sim \pi} \left[ \sum_{t=0}^{\infty} \gamma^t r(s_t,a_t) \mid s_0 = s, a_0 = a \right]$, and the advantage function as $A^{\pi}(s,a) = Q^{\pi}(s,a) - V^{\pi}(s)$.

\paragraph{Policy improvement lower bound}

The starting point in the design of many popular on-policy algorithms is the following policy improvement lower bound, which was first developed by \citet{kakade_2002} and later refined by \citet{schulman_2015} and \citet{achiam_2017}:

\begin{lemma}[\citet{achiam_2017}]\label{lem:pd_lb}
Consider a current policy $\pi_k$. For any future policy $\pi$, we have
\begin{equation}
J(\pi) - J(\pi_k) \geq \frac{1}{1-\gamma} \mathop{\mathbb{E}}_{(s,a) \sim d^{\pi_k}} \left[ \frac{\pi(a \mid s)}{\pi_k(a \mid s)} A^{\pi_k}(s,a) \right] - \frac{2\gamma C^{\pi,\pi_k}}{(1-\gamma)^2} \mathop{\mathbb{E}}_{s \sim d^{\pi_k}} \left[ \operatorname{TV}(\pi,\pi_k)(s) \right],
\end{equation}
where $C^{\pi,\pi_k} = \max_{s \in \mathcal{S}} \left| \mathop{\mathbb{E}}_{a \sim \pi(\cdot \mid s)} \left[ A^{\pi_k}(s,a) \right] \right|$ and $\operatorname{TV}(\pi,\pi_k)(s)$ represents the total variation distance between the distributions $\pi(\, \cdot \mid s)$ and $\pi_k(\, \cdot \mid s)$. 
\end{lemma}

We refer to the first term of the lower bound in Lemma~\ref{lem:pd_lb} as the surrogate objective, and the second term as the penalty term. Note that we can guarantee policy improvement at every step of the learning process by choosing the next policy $\pi_{k+1}$ to maximize this lower bound. Because the expectations in Lemma~\ref{lem:pd_lb} depend on the current policy $\pi_k$, we can approximate this lower bound using samples generated by the current policy.

\paragraph{Proximal Policy Optimization}

PPO, which has become the default on-policy policy optimization algorithm due to its strong performance and simple implementation, is theoretically motivated by the policy improvement lower bound in Lemma~\ref{lem:pd_lb}. Rather than directly maximizing this lower bound, PPO considers the goal of maximizing the surrogate objective while constraining the next policy to be close to the current policy. In particular, PPO heuristically accomplishes this by considering the following objective at every policy update:
\begin{equation}\label{eq:ppo_loss}
L^{\textnormal{PPO}}_k \left( \pi \right) = \mathop{\mathbb{E}}_{(s,a) \sim d^{\pi_k}} \left[ \min \left(  \frac{\pi(a \mid s)}{\pi_k(a \mid s)} A^{\pi_k}(s,a),  \operatorname{clip}\left( \frac{\pi(a \mid s)}{\pi_k(a \mid s)},1-\epsilon,1+\epsilon \right) A^{\pi_k}(s,a) \right) \right],
\end{equation}
where $\operatorname{clip}(x,l,u) = \min (\max (x, l), u)$. As seen in the second term of this objective, PPO constrains the difference between consecutive policies by removing the incentive for the probability ratio $\nicefrac{\pi(a \mid s)}{\pi_k(a \mid s)}$ to leave the clipping range $[1-\epsilon,1+\epsilon]$. Finally, the outer minimization guarantees that \eqref{eq:ppo_loss} is a lower bound to the surrogate objective in Lemma~\ref{lem:pd_lb}. In practice, \eqref{eq:ppo_loss} is approximated using samples generated by the current policy $\pi_k$, and the resulting empirical objective is approximately optimized at every policy update using minibatch stochastic gradient ascent. 

For a sufficiently small learning rate and sufficiently large number of samples, PPO results in stable policy improvement throughout the learning process. However, it is well-known that high variance is a major issue in reinforcement learning, so often the number of samples must be large in order for the empirical objective to be an accurate estimator of the true objective \eqref{eq:ppo_loss}. Because these samples must be collected under the current policy between every policy update, PPO can be very sample intensive.

%%%%%%%%%%%%%%%%%%%%%%%%%%%%%%%%%%%%%%%%%%%%%%%%%%%%%%%%%%%%

\section{Generalized policy improvement lower bound}

A logical approach to improve the sample efficiency of PPO is to reuse samples from prior policies, as done in off-policy algorithms. Unfortunately, the distribution shift between policies invalidates the policy improvement lower bound in Lemma~\ref{lem:pd_lb}, which provides the theoretical support for PPO's reliable performance. In order to retain the stability benefits of PPO while reusing samples from prior policies, we must incorporate these off-policy samples in a principled way. We accomplish this by developing a generalized policy improvement lower bound that can be approximated using samples from the last $M$ policies, rather than requiring samples be generated only from the current policy $\pi_k$:

\begin{theorem}[Generalized Policy Improvement Lower Bound]\label{thm:pd_lb_pimix}
Consider prior policies $\pi_{k-i}$, $i=0,\ldots,M-1$, where $\pi_k$ represents the current policy. For any choice of distribution $\nu = \left[ \nu_0 \cdots \nu_{M-1} \right]$ over the prior $M$ policies and any future policy $\pi$, we have
\begin{multline}
J(\pi) - J(\pi_k) \geq \frac{1}{1-\gamma} \mathop{\mathbb{E}}_{i \sim \nu} \left[ \mathop{\mathbb{E}}_{(s,a) \sim d^{\pi_{k-i}}} \left[ \frac{\pi(a \mid s)}{\pi_{k-i}(a \mid s)} A^{\pi_k}(s,a) \right] \right] \\ - \frac{2 \gamma C^{\pi,\pi_k}}{(1-\gamma)^2} \mathop{\mathbb{E}}_{i \sim \nu} \left[  \mathop{\mathbb{E}}_{s \sim d^{\pi_{k-i}}} \left[ \operatorname{TV}(\pi,\pi_{k-i})(s) \right] \right],
\end{multline}
where $C^{\pi,\pi_k}$ and $\operatorname{TV}(\pi,\pi_{k-i})(s)$ are defined as in Lemma~\ref{lem:pd_lb}. 
\end{theorem}

\begin{proof}
We generalize Lemma~\ref{lem:pd_lb} to depend on expectations with respect to any reference policy, and we apply this result $M$ times where the reference policy is each of $\pi_{k-i}, i=0,\ldots,M-1$, respectively. Then, the convex combination determined by $\nu$ of the resulting $M$ policy improvement lower bounds is also a lower bound. See the Appendix for a full proof.
\end{proof}
Because the expectations in Theorem~\ref{thm:pd_lb_pimix} depend on distributions related to the last $M$ policies, this lower bound provides theoretical support for extending PPO to include off-policy samples. Note that Theorem~\ref{thm:pd_lb_pimix} is still a lower bound on the policy improvement between the current policy $\pi_k$ and a future policy $\pi$. This is true because the advantage function in the surrogate objective and the constant in the penalty term still depend on the current policy $\pi_k$. However, the visitation distribution, probability ratio in the surrogate objective, and total variation distance in the penalty term now depend on prior policies. Finally, the standard policy improvement lower bound in Lemma~\ref{lem:pd_lb} can be recovered from Theorem~\ref{thm:pd_lb_pimix} by setting $M=1$.

The penalty term in Theorem~\ref{thm:pd_lb_pimix} suggests that we should control the expected total variation distances between the future policy $\pi$ and the last $M$ policies. By applying the triangle inequality for total variation distance to each component of the penalty term, we see that 
\begin{multline}\label{eq:tvpen_triangle}
\mathop{\mathbb{E}}_{i \sim \nu} \left[ \mathop{\mathbb{E}}_{s \sim d^{\pi_{k-i}}} \left[ \operatorname{TV}(\pi,\pi_{k-i})(s) \right] \right] \leq \mathop{\mathbb{E}}_{i \sim \nu} \left[ \mathop{\mathbb{E}}_{s \sim d^{\pi_{k-i}}} \left[ \operatorname{TV}(\pi,\pi_{k})(s) \right] \right] \\ + \sum_{j=1}^{M-1} \sum_{i=j}^{M-1} \nu_i \mathop{\mathbb{E}}_{s \sim d^{\pi_{k-i}}} \left[ \operatorname{TV}(\pi_{k-j+1},\pi_{k-j})(s) \right].
\end{multline}

The first term on the right-hand side of \eqref{eq:tvpen_triangle} represents an expected total variation distance between the current policy $\pi_k$ and the future policy $\pi$, while each component of the second term represents an expected total variation distance between consecutive prior policies. This demonstrates that we can effectively control the expected performance loss at every policy update by controlling the expected total variation distance between consecutive policies. We see next that the clipping mechanism in PPO approximately accomplishes this task.

%%%%%%%%%%%%%%%%%%%%%%%%%%%%%%%%%%%%%%%%%%%%%%%%%%%%%%%%%%%%

\section{Clipping mechanism}

\paragraph{Connection to penalty term}

As discussed previously, the clipping mechanism present in the PPO objective removes the incentive for the probability ratio $\nicefrac{\pi(a \mid s)}{\pi_k(a \mid s)}$ to leave the clipping range $[1-\epsilon,1+\epsilon]$. Written differently, the clipping mechanism removes the incentive for the magnitude of
\begin{equation}\label{eq:polrat_mag}
\left| \frac{\pi(a \mid s)}{\pi_k(a \mid s)} - 1  \right|
\end{equation}
to exceed $\epsilon$. We now see that \eqref{eq:polrat_mag} is closely related to the penalty term of the standard policy improvement lower bound in Lemma~\ref{lem:pd_lb} as follows:

\begin{lemma}\label{lem:tv_polrat_pik}
The expected total variation distance between the current policy $\pi_k$ and the future policy $\pi$ that appears in Lemma~\ref{lem:pd_lb} can be rewritten as
\begin{equation}
\mathop{\mathbb{E}}_{s \sim d^{\pi_k}} \left[ \operatorname{TV}(\pi,\pi_k)(s) \right] = \frac{1}{2} \mathop{\mathbb{E}}_{(s,a) \sim d^{\pi_k}} \left[ \, \left| \frac{\pi(a \mid s)}{\pi_k(a \mid s)} - 1  \right| \, \right].
\end{equation}
\end{lemma}

\begin{proof}
See the Appendix.
\end{proof}

Therefore, the clipping mechanism in PPO can be viewed as a heuristic that controls the magnitude of a sample-based approximation of the expectation on the right-hand side of Lemma~\ref{lem:tv_polrat_pik}. It accomplishes this by removing the incentive for \eqref{eq:polrat_mag} to exceed $\epsilon$ at all state-action pairs sampled from the state-action visitation distribution $d^{\pi_k}$. As a result, the clipping mechanism in PPO approximately bounds the expected total variation distance between the current policy $\pi_k$ and the future policy $\pi$ by $\nicefrac{\epsilon}{2}$.

\paragraph{Generalized clipping mechanism for off-policy data}

We can use this connection between the clipping mechanism and penalty term in PPO to derive a generalized clipping mechanism suitable for the off-policy setting. In particular, the decomposition of the off-policy penalty term from Theorem~\ref{thm:pd_lb_pimix} that appears in \eqref{eq:tvpen_triangle} suggests that we should control the expected total variation distance
\begin{equation}\label{eq:tvpen_cur}
\mathop{\mathbb{E}}_{i \sim \nu} \left[ \mathop{\mathbb{E}}_{s \sim d^{\pi_{k-i}}} \left[ \operatorname{TV}(\pi,\pi_{k})(s) \right] \right]
\end{equation}
at each policy update, which can be rewritten as follows:

\begin{lemma}\label{lem:tv_polrat_pimix}
The expected total variation distance between the current policy $\pi_k$ and the future policy $\pi$ in \eqref{eq:tvpen_cur} can be rewritten as
\begin{equation}
\mathop{\mathbb{E}}_{i \sim \nu} \left[ \mathop{\mathbb{E}}_{s \sim d^{\pi_{k-i}}} \left[ \operatorname{TV}(\pi,\pi_{k})(s) \right] \right] = \frac{1}{2} \mathop{\mathbb{E}}_{i \sim \nu} \left[ \mathop{\mathbb{E}}_{(s,a) \sim d^{\pi_{k-i}}} \left[ \, \left| \frac{\pi(a \mid s)}{\pi_{k-i}(a \mid s)} - \frac{\pi_k(a \mid s)}{\pi_{k-i}(a \mid s)}  \right| \, \right] \right].
\end{equation}
\end{lemma}

\begin{proof}
Apply the same techniques as in the proof of Lemma~\ref{lem:tv_polrat_pik}. See the Appendix for details.
\end{proof}

The right-hand side of Lemma~\ref{lem:tv_polrat_pimix} provides insight into the appropriate clipping mechanism to be applied in the off-policy setting in order to approximately control the penalty term in the generalized policy improvement lower bound:

\begin{definition}[Generalized Clipping Mechanism]\label{def:genclip}
Consider a current policy $\pi_k$ and clipping parameter $\epsilon$. For a state-action pair generated using a prior policy $\pi_{k-i}$, the \emph{generalized clipping mechanism} is defined as
\begin{equation}
\operatorname{clip}\left( \frac{\pi(a \mid s)}{\pi_{k-i}(a \mid s)},\frac{\pi_k(a \mid s)}{\pi_{k-i}(a \mid s)} - \epsilon,\frac{\pi_k(a \mid s)}{\pi_{k-i}(a \mid s)} + \epsilon \right).
\end{equation}
\end{definition}

Note that the probability ratio for each state-action pair in the off-policy setting begins in the center of the clipping range for every policy update just as in PPO, where now the center is given by $\nicefrac{\pi_k(a \mid s)}{\pi_{k-i}(a \mid s)}$. Also note that we recover the standard clipping mechanism used in PPO when samples are generated by the current policy $\pi_k$.

\paragraph{Impact of learning rate on clipping mechanism}

Due to the heuristic nature of the clipping mechanism, it can only approximately bound the penalty term in the corresponding policy improvement lower bound if the learning rate used for policy updates is sufficiently small. To see why this is true, note that the clipping mechanism has no impact at the beginning of each policy update since each probability ratio begins at the center of the clipping range \citep{engstrom_2020}. If the learning rate is too large, the initial gradient steps of the policy update can result in probability ratios that are far outside the clipping range. In addition, the sensitivity of the probability ratio to gradient updates can change as training progresses, which suggests that the learning rate may need to change over time in order for the clipping mechanism to approximately enforce a total variation distance trust region throughout the course of training.

In order to address these issues, we propose a simple adaptive learning rate method that is directly connected to our goal of controlling a total variation distance penalty term via the clipping mechanism. Using Lemma~\ref{lem:tv_polrat_pimix}, we can approximate the expected total variation distance of interest using a sample-based estimate. We reduce the learning rate if the estimated total variation distance exceeds our goal of $\nicefrac{\epsilon}{2}$, and we increase the learning rate if the estimated total variation distance is significantly lower than $\nicefrac{\epsilon}{2}$. This approach more closely connects the implementation of PPO to the policy improvement lower bound on which the algorithm is based. In addition, the adaptive learning rate prevents large policy updates that can lead to instability, while also increasing the speed of learning when policy updates are too small. We formally describe this method in Algorithm \ref{alg:geppo}.

%%%%%%%%%%%%%%%%%%%%%%%%%%%%%%%%%%%%%%%%%%%%%%%%%%%%%%%%%%%%

\section{Algorithm}

The surrogate objective from the generalized policy improvement lower bound in Theorem~\ref{thm:pd_lb_pimix}, coupled with the generalized clipping mechanism in Definition~\ref{def:genclip} that controls the penalty term in Theorem~\ref{thm:pd_lb_pimix}, motivate the following generalized PPO objective that directly considers the use of off-policy data:
\begin{multline}\label{eq:geppo_loss}
L^{\textnormal{GePPO}}_k \left( \pi \right) = \mathop{\mathbb{E}}_{i \sim \nu} \Bigg[ \mathop{\mathbb{E}}_{(s,a) \sim d^{\pi_{k-i}}} \left[ \min \left(  \frac{\pi(a \mid s)}{\pi_{k-i}(a \mid s)} A^{\pi_k}(s,a), \right. \right. \\ \left. \left.  \operatorname{clip}\left( \frac{\pi(a \mid s)}{\pi_{k-i}(a \mid s)},\frac{\pi_k(a \mid s)}{\pi_{k-i}(a \mid s)}-\epsilon,\frac{\pi_k(a \mid s)}{\pi_{k-i}(a \mid s)}+\epsilon \right) A^{\pi_k}(s,a) \right) \right] \Bigg].
\end{multline}

Our algorithm, which we call Generalized Proximal Policy Optimization with Sample Reuse (GePPO), approximates this objective using samples collected from each of the last $M$ policies and approximately optimizes it using minibatch stochastic gradient ascent. In addition, we update the learning rate at every iteration using the adaptive method described in the previous section. GePPO, which we detail in Algorithm \ref{alg:geppo}, represents a principled approach to improving the sample efficiency of PPO while retaining its approximate policy improvement guarantees.

\begin{algorithm}[t]
\KwIn{initial policy $\pi_0$; number of prior policies $M$; policy weights $\nu$; clipping parameter $\epsilon$; batch size $n$; initial learning rate $\eta$; adaptive factor $\alpha \geq 0$; minimum threshold factor $0 \leq \beta \leq 1$.}
\BlankLine
\For{$k=0,1,2,\ldots$}{
	\BlankLine
	Collect $n$ samples with $\pi_k$.	
	\BlankLine
	\underline{Update policy}:
	\BlankLine
	Use $n$ samples from each of $\pi_{k-i}$, $i=0,\ldots,M-1$, to approximate $L^{\textnormal{GePPO}}_k \left( \pi \right)$. 
	\BlankLine	
	Approximately maximize the empirical objective using minibatch stochastic gradient ascent.
	\BlankLine
	\underline{Update learning rate}:
	\BlankLine
	Calculate sample-based estimate $\widehat{\operatorname{TV}}$ of expected total variation distance in Lemma~\ref{lem:tv_polrat_pimix}.
	\BlankLine
	\lIf{$\widehat{\operatorname{TV}} > \nicefrac{\epsilon}{2}$}{
		$\eta = \eta \cdot \nicefrac{1}{(1+\alpha)}$
	}
	\BlankLine	
	\DontPrintSemicolon
	\lElseIf{$\widehat{\operatorname{TV}} < \nicefrac{\beta \epsilon}{2}$}{
		$\eta = \eta \cdot (1+\alpha)$.	
	}
	\BlankLine	
}

\caption{Generalized Proximal Policy Optimization with Sample Reuse (GePPO)}\label{alg:geppo}
\end{algorithm}

%%%%%%%%%%%%%%%%%%%%%%%%%%%%%%%%%%%%%%%%%%%%%%%%%%%%%%%%%%%%

\section{Sample efficiency analysis}\label{sec:analysis}

\paragraph{Generalized clipping parameter}

In order to compare GePPO to PPO, we first must determine the appropriate choice of clipping parameter $\epsilon^{\textnormal{GePPO}}$ that results in the same worst-case expected performance loss at every update:

\begin{lemma}\label{lem:geppo_param}
Consider PPO with clipping parameter $\epsilon^{\textnormal{PPO}}$ and GePPO with clipping parameter $\epsilon^{\textnormal{GePPO}}$. If 
\begin{equation}
\epsilon^{\textnormal{GePPO}} = \frac{\epsilon^{\textnormal{PPO}}}{ \mathop{\mathbb{E}}_{i \sim \nu} \left[ \, i + 1 \, \right]},
\end{equation}
then the worst-case expected performance loss at every update is the same under both algorithms.
\end{lemma}

\begin{proof}
See the Appendix.
\end{proof}

Lemma~\ref{lem:geppo_param} shows that we must make smaller policy updates compared to PPO in terms of total variation distance, which is a result of utilizing samples from prior policies. However, these additional samples stabilize policy updates by increasing the batch size used to approximate the true objective, which allows us to make policy updates more frequently. Ultimately, this trade-off results in faster and more stable learning, as we detail next.

\paragraph{Balancing stability and sample efficiency}

For concreteness, in this section we restrict our attention to uniform policy weights over the last $M$ policies, i.e., $\nu_i = \nicefrac{1}{M}$ for $i=0,\ldots,M-1$. We provide additional details in the Appendix on how these policy weights can be optimized to further improve upon the results shown for the uniform case.

Assume we require $N = Bn$ samples for the empirical objective to be a sufficiently accurate estimate of the true objective, where $n$ is the smallest possible batch size we can collect and $B$ is some positive integer. In this setting, PPO makes one policy update per $N$ samples collected, while GePPO leverages data from prior policies to make $B$ updates per $N$ samples collected as long as $M \geq B$. First, we see that GePPO can increase the change in total variation distance of the policy throughout training compared to PPO without sacrificing stability in terms of sample size:

\begin{theorem}\label{thm:geppo_tv}
Set $M = B$ and consider uniform policy weights. Then, GePPO increases the change in total variation distance of the policy throughout training by a factor of $\nicefrac{2B}{(B+1)}$ compared to PPO, while using the same number of samples for each policy update as PPO.
\end{theorem}
\begin{proof}
See the Appendix.
\end{proof}

Alternatively, we see that GePPO can increase the sample size used to approximate each policy update compared to PPO, while maintaining the same change in total variation distance throughout training:

\begin{theorem}\label{thm:geppo_ess}
Set $M = 2B-1$ and consider uniform policy weights. Then, GePPO increases the sample size used for each policy update by a factor of $\nicefrac{(2B-1)}{B}$ compared to PPO, while maintaining the same change in total variation distance of the policy throughout training as PPO.
\end{theorem}
\begin{proof}
See the Appendix.
\end{proof}

By combining the results in Theorem~\ref{thm:geppo_tv} and Theorem~\ref{thm:geppo_ess}, we see that GePPO with uniform policy weights improves the trade-off between stability and sample efficiency in PPO for any choice of $B \leq M \leq 2B-1$. Also note that the benefit of GePPO increases with $B$, where a larger value of $B$ indicates a more complex problem that requires additional samples to estimate the true objective at every policy update. This is precisely the scenario where the trade-off between stability and sample efficiency becomes critical.

%%%%%%%%%%%%%%%%%%%%%%%%%%%%%%%%%%%%%%%%%%%%%%%%%%%%%%%%%%%%

\section{Experiments}\label{sec:experiments}

\begin{table}
  \caption{Performance comparison across MuJoCo tasks.}
  \label{tab:exp}
  \centering  
  \begin{tabular}{l R{0.075} R{0.075} R{0.05} R{0.075} R{0.075} R{0.05} R{0.075} R{0.075} }    
    \toprule
     & \multicolumn{3}{E{0.2}}{Average Performance Over 1M Steps} & \multicolumn{3}{E{0.2}}{Final Performance} & \multicolumn{2}{D{0.15}}{Steps (M) to Final PPO Performance} \\     
    \cmidrule(lr){2-4} \cmidrule(lr){5-7} \cmidrule(lr){8-9}
    Environment & \multicolumn{1}{C{0.075}}{PPO} & \multicolumn{1}{C{0.075}}{GePPO} & \multicolumn{1}{C{0.05}}{\%$^*$} & \multicolumn{1}{C{0.075}}{PPO} & \multicolumn{1}{C{0.075}}{GePPO} & \multicolumn{1}{C{0.05}}{\%$^*$} & \multicolumn{1}{C{0.075}}{PPO} & \multicolumn{1}{C{0.075}}{GePPO} \\
    \midrule
    Swimmer-v3     & $98$ \hspace{1em}     & $161$ \hspace{1em}    & 65\% & $136$ \hspace{1em}    & $195$ \hspace{1em}    & 44\% & $1.00$ \hspace{1em} & $0.23$ \hspace{1em} \\
    Hopper-v3      & $2{,}362$ \hspace{1em} & $2{,}544$ \hspace{1em} & 8\%  & $3{,}126$ \hspace{1em} & $3{,}450$ \hspace{1em} & 10\% & $1.00$ \hspace{1em} & $0.41$ \hspace{1em} \\
    HalfCheetah-v3 & $1{,}764$ \hspace{1em} & $2{,}439$ \hspace{1em} & 38\% & $3{,}223$ \hspace{1em} & $3{,}903$ \hspace{1em} & 21\% & $1.00$ \hspace{1em} & $0.54$ \hspace{1em} \\
    Walker2d-v3    & $1{,}817$ \hspace{1em} & $2{,}199$ \hspace{1em} & 21\% & $3{,}041$ \hspace{1em} & $3{,}502$ \hspace{1em} & 15\% & $1.00$ \hspace{1em} & $0.63$ \hspace{1em} \\
    Ant-v3         & $545$ \hspace{1em}     & $762$ \hspace{1em}     & 40\% & $1{,}227$ \hspace{1em} & $1{,}576$ \hspace{1em} & 28\% & $1.00$ \hspace{1em} & $0.79$ \hspace{1em} \\
    Humanoid-v3    & $584$ \hspace{1em}    & $665$ \hspace{1em}    & 14\% & $972$ \hspace{1em}    & $1{,}345$ \hspace{1em} & 38\% & $1.00$ \hspace{1em} & $0.85$ \hspace{1em} \\            
	\addlinespace
    \cmidrule(lr){1-3}
    \multicolumn{9}{l}{\footnotesize{$^*$ Represents percent improvement of GePPO compared to PPO.}} \\
    \bottomrule
  \end{tabular}
\end{table}

\begin{figure}
\centering
\includegraphics[width=1.00\linewidth]{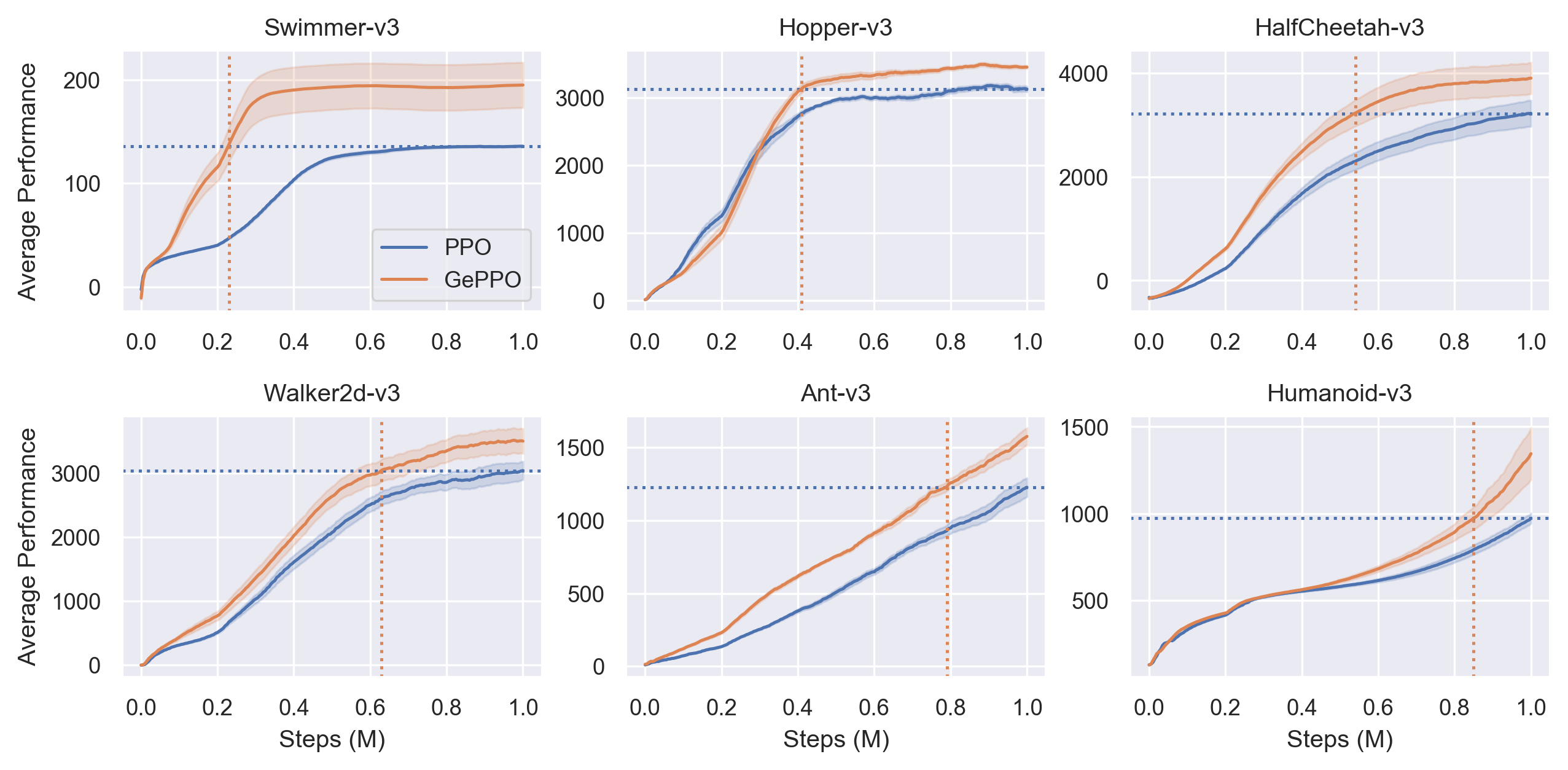}
\caption{Performance throughout training across MuJoCo tasks. Shading denotes half of one standard error. Horizontal dotted lines represent the final performance of PPO, and vertical dotted lines represent the time at which GePPO achieves the same performance.}\label{fig:Jplot}
\end{figure}

In addition to the theoretical support for our algorithm in the previous section, we aim to investigate the stability and sample efficiency of GePPO experimentally through simulations on several MuJoCo environments \citep{todorov_2012} in OpenAI Gym \citep{brockman_2016}. In particular, we consider six continuous control locomotion tasks which vary in dimensionality: Swimmer-v3, Hopper-v3, HalfCheetah-v3, Walker2d-v3, Ant-v3, and Humanoid-v3. We compare the performance of our algorithm to PPO, which is the default on-policy policy optimization algorithm. We do not consider a comparison with popular off-policy algorithms since they lack approximate policy improvement guarantees, and as a result the risk associated with each policy update is not comparable.

We consider the default implementation choices used by \citet{henderson_2018} for PPO. In particular, we represent the policy $\pi$ as a multivariate Gaussian distribution, where the mean action for a given state is parameterized by a neural network with two hidden layers of 64 units each and tanh activations. The state-independent standard deviation is parameterized separately. The default value for the clipping parameter is $\epsilon^{\textnormal{PPO}}=0.2$, and the default batch size is $N=2{,}048$. Sample trajectories for the tasks we consider can contain up to one thousand steps, so we represent the default batch size as $n=1{,}024$ and $B=2$ using the notation from the previous section. 

For GePPO, we select $M$ and the corresponding policy weights $\nu$ to maximize the effective batch size used for policy updates while maintaining the same change in total variation distance throughout training as PPO. The clipping parameter $\epsilon^{\textnormal{GePPO}}$ is chosen according to Lemma~\ref{lem:geppo_param}, which in our experiments results in $\epsilon^{\textnormal{GePPO}}=0.1$. We estimate $A^{\pi_k}(s,a)$ with an off-policy variant of Generalized Advantage Estimation \citep{schulman_2016} that uses the V-trace value function estimator \citep{espeholt_2018}. We run each experiment for a total of one million steps over five random seeds. See the Appendix for additional implementation details, including the values of all hyperparameters.\footnote{Code available at \url{https://github.com/jqueeney/geppo}.} 

As shown in Table~\ref{tab:exp} and Figure~\ref{fig:Jplot}, GePPO results in fast and reliable learning. We see that GePPO leads to improved final performance across all environments. In addition to final performance, we assess the sample efficiency of our algorithm by considering the average performance over the course of training as well as the number of samples required for GePPO to match the final performance of PPO. Despite the fact that the default implementation of PPO achieves stable learning with a small batch size ($B=2$), our results still demonstrate the sample efficiency benefits of GePPO. Compared to PPO, GePPO improves the average performance over training by between 8\% and 65\%. Moreover, GePPO requires between 15\% and 77\% fewer samples to reach the final performance level of PPO.

In addition to improving sample efficiency, GePPO also ensures that the total variation distance between consecutive policies remains close to the target determined by the clipping parameter through the use of an adaptive learning rate. This is not the case in PPO, where we observe that the change in total variation distance per policy update increases throughout training. As shown on the left-hand side of Figure~\ref{fig:wide}, the change in total variation distance for PPO under default settings is almost 40\% higher than desired after one million steps on Walker2d-v3. We demonstrate that this can lead to instability in the training process by considering a standard wide policy network with two hidden layers of 400 and 300 units, respectively. We see on the right-hand side of Figure~\ref{fig:wide} that this minor implementation change exacerbates the trend observed in PPO under the default settings, resulting in unstable learning where performance declines over time due to excessively large policy updates. GePPO, on the other hand, successfully controls this source of instability through its adaptive learning rate, resulting in stable policy improvement that is robust to implementation choices. 

\begin{figure}
\centering
\includegraphics[width=1.0\linewidth]{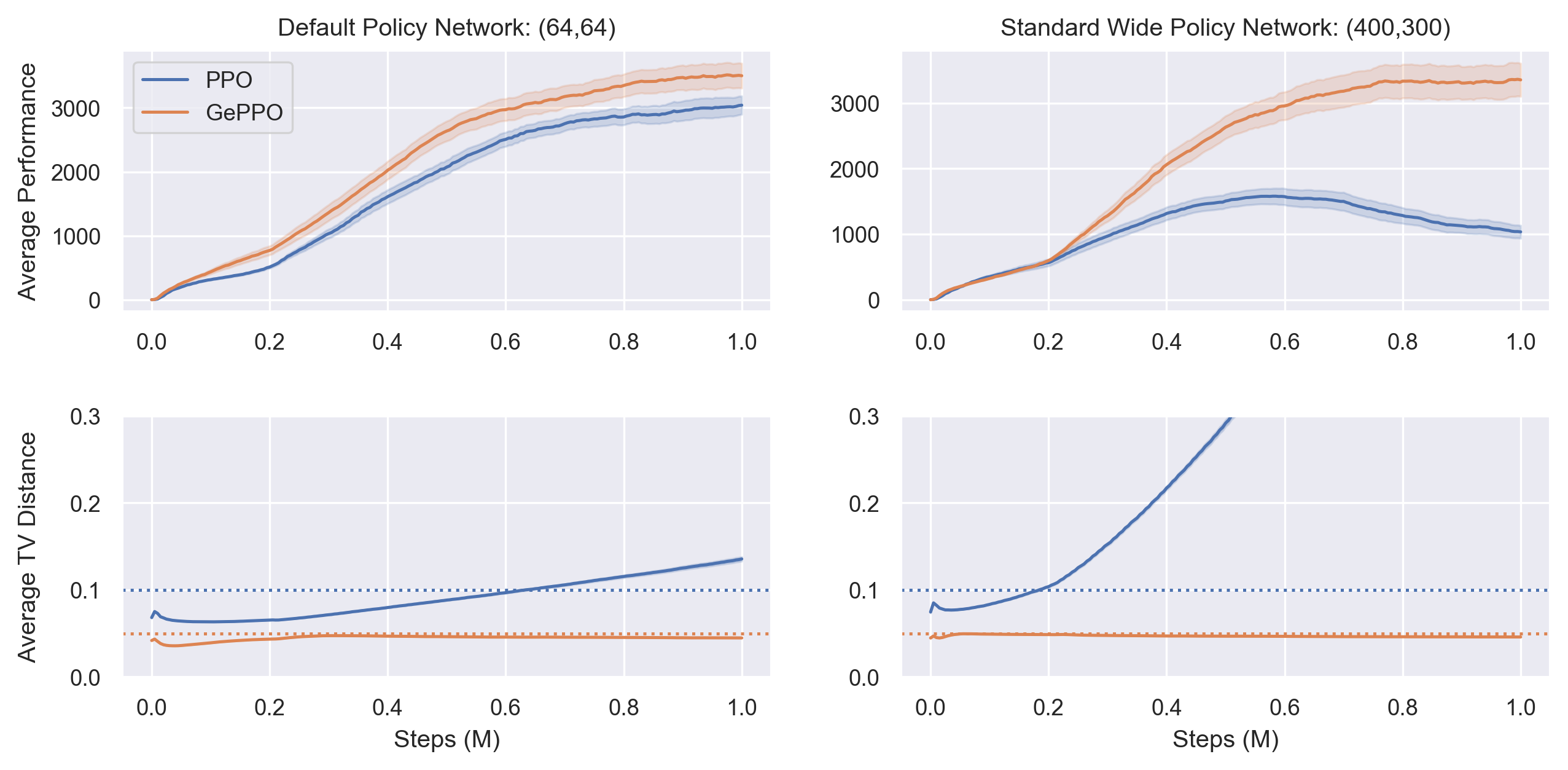}
\caption{Evaluation on Walker2d-v3 with different policy networks. Hidden layer sizes in parentheses. Shading denotes half of one standard error. Top: Performance throughout training. Bottom: Change in average total variation distance per policy update. Horizontal dotted lines represent target total variation distances for PPO and GePPO, respectively.}\label{fig:wide}
\end{figure}

%%%%%%%%%%%%%%%%%%%%%%%%%%%%%%%%%%%%%%%%%%%%%%%%%%%%%%%%%%%%

\section{Conclusion}\label{sec:conclusion}

We have presented a principled approach to incorporating off-policy samples into PPO that is theoretically supported by a novel off-policy policy improvement lower bound. Our algorithm, GePPO, improves the sample efficiency of PPO, and can be viewed as a more reliable approach to sample reuse than standard off-policy algorithms that are not based on approximate policy improvement guarantees. This represents an important step towards developing stable, sample efficient reinforcement learning methods that can be applied in high-stakes real-world decision making.

Despite this progress, there remain limitations that must be addressed in order for reinforcement learning to achieve widespread real-world adoption. Because our algorithm is based on PPO, policy improvement guarantees are only approximately achieved due to the use of the clipping mechanism heuristic. In addition, the constant factor in the penalty term on which our algorithm is based may be too large to deliver practical guarantees. Finally, we considered PPO as our starting point due to its theoretical support and stable performance, but there may exist other approaches that more effectively balance the goals of stability and sample efficiency. These represent interesting avenues for future work in order to develop reinforcement learning methods that can be trusted to improve real-world decision making.

%%%%%%%%%%%%%%%%%%%%%%%%%%%%%%%%%%%%%%%%%%%%%%%%%%%%%%%%%%%%

% Acknowledgments

\begin{ack}
This research was partially supported by the NSF under grants ECCS-1931600, DMS-1664644, CNS-1645681, and IIS-1914792, by the ONR under grants N00014-19-1-2571 and N00014-21-1-2844, by the NIH under grants R01 GM135930 and UL54 TR004130, by the DOE under grants DE-AR-0001282 and NETL-EE0009696, by AFOSR under grant FA9550-19-1-0158, by the MathWorks, by the Boston University Kilachand Fund for Integrated Life Science and Engineering, and by the Qatar National Research Fund, a member of the Qatar Foundation (the statements made herein are solely the responsibility of the authors).
\end{ack}

%%%%%%%%%%%%%%%%%%%%%%%%%%%%%%%%%%%%%%%%%%%%%%%%%%%%%%%%%%%%

% References

\bibliography{NeurIPS2021_GePPO_bib}

%%%%%%%%%%%%%%%%%%%%%%%%%%%%%%%%%%%%%%%%%%%%%%%%%%%%%%%%%%%%

\newpage

\section*{Appendix}

\appendix

%%%%%%%%%%%%%%%%%%%%%%%%%%%%%%%%%%%%%%%%%%%%%%%%%%%%%%%%%%%%

\section{Proofs}\label{sec:app_proofs}

\subsection{Proof of Theorem~\ref{thm:pd_lb_pimix}}

We begin by stating results from \citet{kakade_2002} and \citet{achiam_2017} that we will use in our proof.

\begin{lemma}[\citet{kakade_2002}]\label{lem:pd_equal}
Consider a current policy $\pi_k$. For any future policy $\pi$, we have
\begin{equation}
J(\pi) - J(\pi_k) = \frac{1}{1-\gamma} \mathop{\mathbb{E}}_{s \sim d^{\pi}} \left[ \mathop{\mathbb{E}}_{a \sim \pi(\cdot \mid s)} \left[ \,  A^{\pi_k}(s,a) \, \right] \right].
\end{equation}
\end{lemma}

\begin{lemma}[\citet{achiam_2017}]\label{lem:tv_stat}
Consider a reference policy $\pi_{\textnormal{ref}}$ and a future policy $\pi$. Then, the total variation distance between the state visitation distributions $d^{\pi_{\textnormal{ref}}}$ and $d^\pi$ is bounded by
\begin{equation}
\operatorname{TV}(d^\pi,d^{\pi_{\textnormal{ref}}}) \leq \frac{\gamma}{1-\gamma} \mathop{\mathbb{E}}_{s \sim d^{\pi_{\textnormal{ref}}}} \left[ \operatorname{TV}(\pi,\pi_{\textnormal{ref}})(s) \right],
\end{equation}
where $\operatorname{TV}(\pi,\pi_{\textnormal{ref}})(s)$ is defined as in Lemma~\ref{lem:pd_lb}.
\end{lemma}

Using these results, we first generalize Lemma~\ref{lem:pd_lb} to depend on expectations with respect to any reference policy:

\begin{lemma}\label{lem:pd_lb_piref}
Consider a current policy $\pi_k$, and any reference policy $\pi_{\textnormal{ref}}$. For any future policy $\pi$, we have
\begin{equation}\label{eq:pd_lb_piref}
J(\pi) - J(\pi_k) \geq \frac{1}{1-\gamma} \mathop{\mathbb{E}}_{(s,a) \sim d^{\pi_{\textnormal{ref}}}} \left[ \frac{\pi(a \mid s)}{\pi_{\textnormal{ref}}(a \mid s)} A^{\pi_k}(s,a) \right] - \frac{2\gamma C^{\pi,\pi_k}}{(1-\gamma)^2} \mathop{\mathbb{E}}_{s \sim d^{\pi_{\textnormal{ref}}}} \left[ \operatorname{TV}(\pi,\pi_{\textnormal{ref}})(s) \right],
\end{equation}
where $C^{\pi,\pi_k}$ and $\operatorname{TV}(\pi,\pi_{\textnormal{ref}})(s)$ are defined as in Lemma~\ref{lem:pd_lb}. 
\end{lemma}

\begin{proof}
The proof is similar to the proof of Lemma~\ref{lem:pd_lb} in \citet{achiam_2017}. Starting from the equality in Lemma~\ref{lem:pd_equal}, we add and subtract the term
\begin{equation}
\frac{1}{1-\gamma} \mathop{\mathbb{E}}_{s \sim d^{\pi_{\textnormal{ref}}}} \left[ \mathop{\mathbb{E}}_{a \sim \pi(\cdot \mid s)} \left[ \,  A^{\pi_k}(s,a) \, \right] \right].
\end{equation}
By doing so, we have
\begin{equation}\label{eq:pd_lb_piref_start}
\begin{split}
J(\pi) - J(\pi_k) &= \frac{1}{1-\gamma} \mathop{\mathbb{E}}_{s \sim d^{\pi_{\textnormal{ref}}}} \left[ \mathop{\mathbb{E}}_{a \sim \pi(\cdot \mid s)} \left[ \,  A^{\pi_k}(s,a) \, \right] \right] \\
& \qquad + \frac{1}{1-\gamma} \left(  \mathop{\mathbb{E}}_{s \sim d^{\pi}} \left[ \mathop{\mathbb{E}}_{a \sim \pi(\cdot \mid s)} \left[ \,  A^{\pi_k}(s,a) \, \right] \right] - \mathop{\mathbb{E}}_{s \sim d^{\pi_{\textnormal{ref}}}} \left[ \mathop{\mathbb{E}}_{a \sim \pi(\cdot \mid s)} \left[ \,  A^{\pi_k}(s,a) \, \right] \right] \right) \\
&\geq \frac{1}{1-\gamma} \mathop{\mathbb{E}}_{s \sim d^{\pi_{\textnormal{ref}}}} \left[ \mathop{\mathbb{E}}_{a \sim \pi(\cdot \mid s)} \left[ \,  A^{\pi_k}(s,a) \, \right] \right] \\
& \qquad - \frac{1}{1-\gamma} \left|  \mathop{\mathbb{E}}_{s \sim d^{\pi}} \left[ \mathop{\mathbb{E}}_{a \sim \pi(\cdot \mid s)} \left[ \,  A^{\pi_k}(s,a) \, \right] \right] - \mathop{\mathbb{E}}_{s \sim d^{\pi_{\textnormal{ref}}}} \left[ \mathop{\mathbb{E}}_{a \sim \pi(\cdot \mid s)} \left[ \,  A^{\pi_k}(s,a) \, \right] \right] \right|.
\end{split}
\end{equation}
We can bound the second term in \eqref{eq:pd_lb_piref_start} using H{\"o}lder's inequality:
\begin{multline}
\frac{1}{1-\gamma} \left|  \mathop{\mathbb{E}}_{s \sim d^{\pi}} \left[ \mathop{\mathbb{E}}_{a \sim \pi(\cdot \mid s)} \left[ \,  A^{\pi_k}(s,a) \, \right] \right] - \mathop{\mathbb{E}}_{s \sim d^{\pi_{\textnormal{ref}}}} \left[ \mathop{\mathbb{E}}_{a \sim \pi(\cdot \mid s)} \left[ \,  A^{\pi_k}(s,a) \, \right] \right] \right| \\ \leq \frac{1}{1-\gamma} \left\Vert d^{\pi} - d^{\pi_{\textnormal{ref}}} \right\Vert_1 \left\Vert \mathop{\mathbb{E}}_{a \sim \pi(\cdot \mid s)} \left[ \,  A^{\pi_k}(s,a) \, \right] \right\Vert_\infty,
\end{multline}
where $d^{\pi}$ and $d^{\pi_{\textnormal{ref}}}$ represent state visitation distributions. From the definition of total variation distance and Lemma~\ref{lem:tv_stat}, we have
\begin{equation}
\left\Vert d^{\pi} - d^{\pi_{\textnormal{ref}}} \right\Vert_1 = 2 \, \operatorname{TV}(d^\pi,d^{\pi_{\textnormal{ref}}}) \leq \frac{2\gamma}{1-\gamma} \mathop{\mathbb{E}}_{s \sim d^{\pi_{\textnormal{ref}}}} \left[ \operatorname{TV}(\pi,\pi_{\textnormal{ref}})(s) \right].
\end{equation}
Also note that
\begin{equation}
\left\Vert \mathop{\mathbb{E}}_{a \sim \pi(\cdot \mid s)} \left[ \,  A^{\pi_k}(s,a) \, \right] \right\Vert_\infty = \max_{s \in \mathcal{S}} \left| \mathop{\mathbb{E}}_{a \sim \pi(\cdot \mid s)} \left[ A^{\pi_k}(s,a) \right] \right| = C^{\pi,\pi_k}.
\end{equation}
As a result, we have that
\begin{equation}\label{eq:pd_lb_piref_almost}
J(\pi) - J(\pi_k) \geq \frac{1}{1-\gamma} \mathop{\mathbb{E}}_{s \sim d^{\pi_{\textnormal{ref}}}} \left[ \mathop{\mathbb{E}}_{a \sim \pi(\cdot \mid s)} \left[ \,  A^{\pi_k}(s,a) \, \right] \right] - \frac{2\gamma C^{\pi,\pi_k}}{(1-\gamma)^2} \mathop{\mathbb{E}}_{s \sim d^{\pi_{\textnormal{ref}}}} \left[ \operatorname{TV}(\pi,\pi_{\textnormal{ref}})(s) \right].
\end{equation}
Finally, assume that the support of $\pi$ is contained in the support of $\pi_{\textnormal{ref}}$ for all states, which is true for common policy representations used in policy optimization. Then, we can rewrite the first term on the right-hand side of \eqref{eq:pd_lb_piref_almost} as
\begin{equation}
\frac{1}{1-\gamma} \mathop{\mathbb{E}}_{s \sim d^{\pi_{\textnormal{ref}}}} \left[ \mathop{\mathbb{E}}_{a \sim \pi(\cdot \mid s)} \left[ \,  A^{\pi_k}(s,a) \, \right] \right] = \frac{1}{1-\gamma} \mathop{\mathbb{E}}_{(s,a) \sim d^{\pi_{\textnormal{ref}}}} \left[ \frac{\pi(a \mid s)}{\pi_{\textnormal{ref}}(a \mid s)} A^{\pi_k}(s,a) \right],
\end{equation}
which results in the lower bound in \eqref{eq:pd_lb_piref}.
\end{proof}

We are now ready to prove Theorem~\ref{thm:pd_lb_pimix}:

\begin{proof}[Proof of Theorem~\ref{thm:pd_lb_pimix}]
Consider prior policies $\pi_{k-i}$, $i=0,\ldots,M-1$. For each prior policy, assume that the support of $\pi$ is contained in the support of $\pi_{k-i}$ for all states, which is true for common policy representations used in policy optimization. Then, by Lemma~\ref{lem:pd_lb_piref} we have
\begin{equation}\label{eq:pd_lb_piprior}
J(\pi) - J(\pi_k) \geq \frac{1}{1-\gamma} \mathop{\mathbb{E}}_{(s,a) \sim d^{\pi_{k-i}}} \left[ \frac{\pi(a \mid s)}{\pi_{k-i}(a \mid s)} A^{\pi_k}(s,a) \right] - \frac{2 \gamma C^{\pi,\pi_k}}{(1-\gamma)^2} \mathop{\mathbb{E}}_{s \sim d^{\pi_{k-i}}} \left[ \operatorname{TV}(\pi,\pi_{k-i})(s) \right]
\end{equation}
for each $\pi_{k-i}$, $i=0,\ldots,M-1$. Consider policy weights $\nu = \left[ \nu_0 \cdots \nu_{M-1} \right]$ over the last $M$ policies, where $\nu$ is a distribution. Then, for any choice of distribution $\nu$, the convex combination determined by $\nu$ of the $M$ lower bounds given by \eqref{eq:pd_lb_piprior} results in the lower bound
\begin{multline}
J(\pi) - J(\pi_k) \geq \frac{1}{1-\gamma} \mathop{\mathbb{E}}_{i \sim \nu} \left[ \mathop{\mathbb{E}}_{(s,a) \sim d^{\pi_{k-i}}} \left[ \frac{\pi(a \mid s)}{\pi_{k-i}(a \mid s)} A^{\pi_k}(s,a) \right] \right] \\ - \frac{2 \gamma C^{\pi,\pi_k}}{(1-\gamma)^2} \mathop{\mathbb{E}}_{i \sim \nu} \left[  \mathop{\mathbb{E}}_{s \sim d^{\pi_{k-i}}} \left[ \operatorname{TV}(\pi,\pi_{k-i})(s) \right] \right].
\end{multline}
\end{proof}

\subsection{Proof of Lemma~\ref{lem:tv_polrat_pik}}

\begin{proof}
From the definition of total variation distance, we have that
\begin{equation}
\mathop{\mathbb{E}}_{s \sim d^{\pi_k}} \left[ \operatorname{TV}(\pi,\pi_k)(s) \right] = \mathop{\mathbb{E}}_{s \sim d^{\pi_k}} \left[ \frac{1}{2} \int_{a \in \mathcal{A}} \left| \pi(a \mid s) - \pi_k(a \mid s)  \right| \mathrm{d}a \right].
\end{equation}
Assume that the support of $\pi$ is contained in the support of $\pi_k$ for all states, which is true for common policy representations used in policy optimization. Then, by multiplying and dividing by $\pi_k(a \mid s)$, we see that
\begin{equation}
\begin{split}
\mathop{\mathbb{E}}_{s \sim d^{\pi_k}} \left[ \operatorname{TV}(\pi,\pi_k)(s) \right] &=  \mathop{\mathbb{E}}_{s \sim d^{\pi_k}} \left[ \frac{1}{2} \int_{a \in \mathcal{A}} \pi_k(a \mid s) \left| \frac{\pi(a \mid s)}{\pi_k(a \mid s)} - 1  \right| \mathrm{d}a \right] \\
&= \frac{1}{2} \mathop{\mathbb{E}}_{(s,a) \sim d^{\pi_k}} \left[ \, \left| \frac{\pi(a \mid s)}{\pi_k(a \mid s)} - 1  \right| \, \right].
\end{split}
\end{equation}

\end{proof}

\subsection{Proof of Lemma~\ref{lem:tv_polrat_pimix}}

\begin{proof}
From the definition of total variation distance, we have that
\begin{equation}
\mathop{\mathbb{E}}_{i \sim \nu} \left[ \mathop{\mathbb{E}}_{s \sim d^{\pi_{k-i}}} \left[ \operatorname{TV}(\pi,\pi_{k})(s) \right] \right] = \mathop{\mathbb{E}}_{i \sim \nu} \left[ \mathop{\mathbb{E}}_{s \sim d^{\pi_{k-i}}} \left[ \frac{1}{2} \int_{a \in \mathcal{A}} \left| \pi(a \mid s) - \pi_k(a \mid s)  \right| \mathrm{d}a \right] \right].
\end{equation}
For $i=0,\ldots,M-1$, assume that the supports of $\pi$ and $\pi_k$ are contained in the support of $\pi_{k-i}$ for all states, which is true for common policy representations used in policy optimization. Then, by multiplying and dividing by $\pi_{k-i}(a \mid s)$, we see that
\begin{equation}
\begin{split}
&\mathop{\mathbb{E}}_{i \sim \nu} \left[ \mathop{\mathbb{E}}_{s \sim d^{\pi_{k-i}}} \left[ \operatorname{TV}(\pi,\pi_{k})(s) \right] \right] \\
& \qquad \qquad \qquad =  \mathop{\mathbb{E}}_{i \sim \nu} \left[ \mathop{\mathbb{E}}_{s \sim d^{\pi_{k-i}}} \left[ \frac{1}{2} \int_{a \in \mathcal{A}} \pi_{k-i}(a \mid s) \left| \frac{\pi(a \mid s)}{\pi_{k-i}(a \mid s)} - \frac{\pi_k(a \mid s)}{\pi_{k-i}(a \mid s)}  \right| \mathrm{d}a \right] \right] \\
& \qquad \qquad \qquad = \frac{1}{2} \mathop{\mathbb{E}}_{i \sim \nu} \left[ \mathop{\mathbb{E}}_{(s,a) \sim d^{\pi_{k-i}}} \left[ \, \left| \frac{\pi(a \mid s)}{\pi_{k-i}(a \mid s)} - \frac{\pi_k(a \mid s)}{\pi_{k-i}(a \mid s)}  \right| \, \right] \right].
\end{split}
\end{equation}

\end{proof}

\subsection{Proof of Lemma~\ref{lem:geppo_param}}

\begin{proof}
From Lemma~\ref{lem:tv_polrat_pik}, PPO approximately bounds the penalty term in the standard policy improvement lower bound in Lemma~\ref{lem:pd_lb} by
\begin{equation}\label{eq:ppo_pen_bound}
\frac{2\gamma C^{\pi,\pi_k}}{(1-\gamma)^2} \mathop{\mathbb{E}}_{s \sim d^{\pi_k}} \left[ \operatorname{TV}(\pi,\pi_k)(s) \right] \leq \frac{2\gamma C^{\pi,\pi_k}}{(1-\gamma)^2} \cdot \frac{\epsilon^{\textnormal{PPO}}}{2}.
\end{equation}
Using the triangle inequality for total variation distance, we see that the penalty term in the generalized policy improvement lower bound in Theorem~\ref{thm:pd_lb_pimix} can be bounded by
\begin{multline}\label{eq:geppo_pen_triangle}
\frac{2 \gamma C^{\pi,\pi_k}}{(1-\gamma)^2} \mathop{\mathbb{E}}_{i \sim \nu} \left[  \mathop{\mathbb{E}}_{s \sim d^{\pi_{k-i}}} \left[ \operatorname{TV}(\pi,\pi_{k-i})(s) \right] \right] \\ \leq \frac{2 \gamma C^{\pi,\pi_k}}{(1-\gamma)^2} \mathop{\mathbb{E}}_{i \sim \nu} \left[ \sum_{j=0}^{i} \mathop{\mathbb{E}}_{s \sim d^{\pi_{k-i}}} \left[ \operatorname{TV}(\pi_{k-j+1},\pi_{k-j})(s) \right] \right],
\end{multline}
where we have written the future policy $\pi$ as $\pi_{k+1}$ on the right-hand side. Note that policy updates in GePPO approximately bound each expected total variation distance that appears on the right-hand side of \eqref{eq:geppo_pen_triangle} by $\nicefrac{\epsilon^{\textnormal{GePPO}}}{2}$. Therefore, the penalty term in the generalized policy improvement lower bound is approximately bounded by
\begin{equation}\label{eq:geppo_pen_bound}
\begin{split}
\frac{2 \gamma C^{\pi,\pi_k}}{(1-\gamma)^2} \mathop{\mathbb{E}}_{i \sim \nu} \left[  \mathop{\mathbb{E}}_{s \sim d^{\pi_{k-i}}} \left[ \operatorname{TV}(\pi,\pi_{k-i})(s) \right] \right] & \leq \frac{2 \gamma C^{\pi,\pi_k}}{(1-\gamma)^2} \mathop{\mathbb{E}}_{i \sim \nu} \left[ \frac{\epsilon^{\textnormal{GePPO}}}{2} \cdot (i+1)  \right] \\
&= \frac{2 \gamma C^{\pi,\pi_k}}{(1-\gamma)^2} \cdot \frac{\epsilon^{\textnormal{GePPO}}}{2} \cdot \mathop{\mathbb{E}}_{i \sim \nu} \left[ \, i+1 \, \right].
\end{split}
\end{equation}
By comparing the bounds in \eqref{eq:ppo_pen_bound} and \eqref{eq:geppo_pen_bound}, we see that the worst-case expected performance loss at every update is the same for PPO and GePPO when
\begin{equation}
\epsilon^{\textnormal{GePPO}} = \frac{\epsilon^{\textnormal{PPO}}}{ \mathop{\mathbb{E}}_{i \sim \nu} \left[ \, i + 1 \, \right]}.
\end{equation}
\end{proof}

\subsection{Proof of Theorem~\ref{thm:geppo_tv}}

\begin{proof}
For $M=B$ with uniform policy weights, we see by Lemma~\ref{lem:geppo_param} that
\begin{equation}
\epsilon^{\textnormal{GePPO}} = \frac{\epsilon^{\textnormal{PPO}}}{ \frac{1}{B} \sum_{i=0}^{B-1} \left( i + 1 \right)} = \frac{2}{B+1} \cdot \epsilon^{\textnormal{PPO}}.
\end{equation}
PPO makes one policy update per $N=Bn$ samples collected, which results in a policy change of $\nicefrac{\epsilon^{\textnormal{PPO}}}{2}$ in terms of total variation distance. By leveraging data from prior policies to obtain $N$ samples per update, GePPO makes $B$ policy updates per $N$ samples collected. This results in an overall policy change of
\begin{equation}
B \cdot \frac{\epsilon^{\textnormal{GePPO}}}{2} = \frac{2B}{B+1} \cdot \frac{\epsilon^{\textnormal{PPO}}}{2}
\end{equation}
in terms of total variation distance for every $N$ samples collected. Therefore, GePPO increases the change in total variation distance of the policy throughout training by a factor of $\nicefrac{2B}{(B+1)}$ compared to PPO, while using the same number of samples for each policy update.
\end{proof}

\subsection{Proof of Theorem~\ref{thm:geppo_ess}}

\begin{proof}
Because $M = 2B-1$, GePPO uses $(2B-1)n$ samples to compute each policy update, compared to $N=Bn$ samples used in PPO. Therefore, GePPO increases the sample size used for each policy update by a factor of $\nicefrac{(2B-1)}{B}$ compared to PPO.

For $M=2B-1$ with uniform policy weights, we see by Lemma~\ref{lem:geppo_param} that
\begin{equation}
\epsilon^{\textnormal{GePPO}} = \frac{\epsilon^{\textnormal{PPO}}}{ \frac{1}{2B-1} \sum_{i=0}^{2B-2} \left( i + 1 \right)} = \frac{\epsilon^{\textnormal{PPO}}}{B}.
\end{equation}
As in Theorem~\ref{thm:geppo_tv}, PPO makes one policy update per $N$ samples collected, while GePPO makes $B$ policy updates per $N$ samples collected. This results in an overall change in total variation distance of
\begin{equation}
B \cdot \frac{\epsilon^{\textnormal{GePPO}}}{2} = \frac{\epsilon^{\textnormal{PPO}}}{2},
\end{equation}
which is the same as in PPO.
\end{proof}

%%%%%%%%%%%%%%%%%%%%%%%%%%%%%%%%%%%%%%%%%%%%%%%%%%%%%%%%%%%%

\section{Optimal policy weights}\label{sec:app_weights}

In Section~\ref{sec:analysis}, we demonstrated the benefits of our algorithm with uniform policy weights, i.e., $\nu_i = \nicefrac{1}{M}$ for $i=0,\ldots,M-1$. Because our generalized policy improvement lower bound holds for any choice of policy weights $\nu$, we can improve upon the results in Theorem~\ref{thm:geppo_tv} and Theorem~\ref{thm:geppo_ess} by optimizing $\nu$.

Non-uniform policy weights introduce an additional source of variance, so in order to account for this we must extend the notion of sample size to effective sample size. The effective sample size represents the number of uniformly-weighted samples that result in the same level of variance. For $n$ samples collected under each of the prior $M$ policies, the effective sample size used in the empirical objective of GePPO with policy weights $\nu$ can be written as
\begin{equation}\label{eq:ess}
\operatorname{ESS}^\textnormal{GePPO} = \frac{n}{\sum_{i=0}^{M-1} \nu_i^2}.
\end{equation}
Note that the effective sample size of GePPO with non-uniform policy weights is always smaller than the true number of samples used to calculate the empirical objective, i.e., $\operatorname{ESS}^\textnormal{GePPO} < Mn$ unless $\nu$ is the uniform distribution.

By Lemma~\ref{lem:geppo_param}, GePPO results in an overall policy change of
\begin{equation}\label{eq:tvchange_gen}
B \cdot \frac{\epsilon^{\textnormal{GePPO}}}{2} = \frac{B}{\sum_{i=0}^{M-1} \nu_i(i+1)} \cdot \frac{\epsilon^{\textnormal{PPO}}}{2}
\end{equation}
in terms of total variation distance for every $N=Bn$ samples collected, as long as the effective sample size of GePPO is at least $N$. Using the general forms of effective sample size in \eqref{eq:ess} and total variation distance change in \eqref{eq:tvchange_gen}, we can optimize the results in Theorem~\ref{thm:geppo_tv} as follows:

\begin{theorem}\label{thm:geppo_tvopt}
Let $\bar{M} \geq B$ be the maximum number of prior policies. Consider the goal of maximizing the change in total variation distance of the policy with GePPO while maintaining the same effective sample size of $N=Bn$ used in PPO. Then, the policy weights $\nu$ that achieve this goal are the optimal solution to the convex optimization problem
\begin{equation}\label{eq:geppo_tvopt}
\begin{split}
\min_{\nu_0,\ldots,\nu_{\bar{M}-1}} & \sum_{i=0}^{\bar{M}-1} \nu_i (i+1)  \\
\mathrm{s.t.} \quad & \, \sum_{i=0}^{\bar{M}-1} \nu_i^2 \leq \frac{1}{B}, \quad  \sum_{i=0}^{\bar{M}-1} \nu_i = 1,  \\
 & \, \nu_i \geq 0, \, i=0,\ldots,\bar{M}-1. \\
\end{split}
\end{equation}
\end{theorem}

\begin{proof}
We can maximize the total variation distance change in \eqref{eq:tvchange_gen} by choosing $\nu$ that minimizes the denominator, leading to the objective in \eqref{eq:geppo_tvopt}. Next, we must have an effective sample size that is at least $N=Bn$. Equivalently, we need
\begin{equation}
\operatorname{ESS}^\textnormal{GePPO} = \frac{n}{\sum_{i=0}^{\bar{M}-1} \nu_i^2} \geq Bn,
\end{equation}
which can be rewritten as the first constraint in \eqref{eq:geppo_tvopt}. Finally, the other constraints in \eqref{eq:geppo_tvopt} ensure that $\nu$ is a distribution.
\end{proof}

Theorem~\ref{thm:geppo_tv} considers uniform policy weights with $M=B$, which are a feasible solution to \eqref{eq:geppo_tvopt}. As a result, Theorem~\ref{thm:geppo_tvopt} increases the change in total variation distance compared to Theorem~\ref{thm:geppo_tv}.

\begin{figure}
\centering
\includegraphics[width=1.00\linewidth]{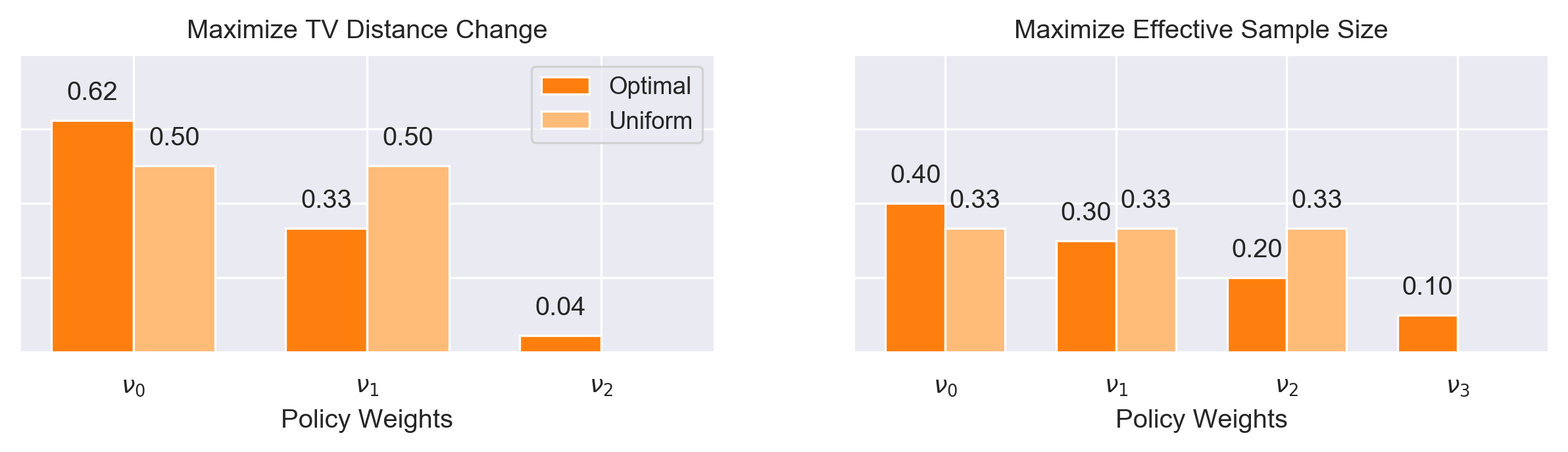}
\caption{Comparison of optimal and uniform policy weights for GePPO when $B=2$. Left: Policy weights determined by Theorem~\ref{thm:geppo_tvopt} and Theorem~\ref{thm:geppo_tv}, respectively. Right: Policy weights determined by Theorem~\ref{thm:geppo_essopt} and Theorem~\ref{thm:geppo_ess}, respectively.}\label{fig:polweights}
\end{figure}

Similarly, we can optimize the results in Theorem~\ref{thm:geppo_ess} as follows:

\begin{theorem}\label{thm:geppo_essopt}
Let $\bar{M} \geq 2B-1$ be the maximum number of prior policies. Consider the goal of maximizing the effective sample size with GePPO while maintaining the same change in total variation distance of the policy throughout training as PPO. Then, the policy weights $\nu$ that achieve this goal are the optimal solution to the convex optimization problem
\begin{equation}\label{eq:geppo_essopt}
\begin{split}
\min_{\nu_0,\ldots,\nu_{\bar{M}-1}} & \sum_{i=0}^{\bar{M}-1} \nu_i^2  \\
\mathrm{s.t.} \quad & \, \sum_{i=0}^{\bar{M}-1} \nu_i (i+1) = B, \quad  \sum_{i=0}^{\bar{M}-1} \nu_i = 1,  \\
 & \, \nu_i \geq 0, \, i=0,\ldots,\bar{M}-1. \\
\end{split}
\end{equation}
\end{theorem}

\begin{proof}
We can maximize the effective sample size in \eqref{eq:ess} by choosing $\nu$ that minimizes the denominator, leading to the objective in \eqref{eq:geppo_essopt}. Next, the total variation distance change in \eqref{eq:tvchange_gen} must equal $\nicefrac{\epsilon^{\textnormal{PPO}}}{2}$, which we accomplish with the first constraint in \eqref{eq:geppo_essopt}. Finally, the other constraints in \eqref{eq:geppo_essopt} ensure that $\nu$ is a distribution.
\end{proof}

Theorem~\ref{thm:geppo_ess} considers uniform policy weights with $M=2B-1$, which are a feasible solution to \eqref{eq:geppo_essopt}. Therefore, Theorem~\ref{thm:geppo_essopt} increases the effective sample size compared to Theorem~\ref{thm:geppo_ess}. Also note that by selecting $\bar{M}$ to be large, Theorem~\ref{thm:geppo_tvopt} and Theorem~\ref{thm:geppo_essopt} solve for both the optimal choice of $M$ and the corresponding optimal weights.

See Figure~\ref{fig:polweights} for a comparison of the optimal and uniform policy weights when $B=2$, which is the setting for the experiments we considered. Note that the optimal policy weights can be found very efficiently, and the convex optimization problems in Theorem~\ref{thm:geppo_tvopt} and Theorem~\ref{thm:geppo_essopt} only need to be solved once at the beginning of training to determine the policy weights $\nu$. 

%%%%%%%%%%%%%%%%%%%%%%%%%%%%%%%%%%%%%%%%%%%%%%%%%%%%%%%%%%%%

\section{Implementation details}\label{sec:app_impl} 

\begin{table}
  \caption{Hyperparameter values for experimental results.}
  \label{tab:hyperparameters}
  \centering
  \begin{tabular}{L{0.4}  R{0.12} R{0.12} R{0.12} }  
    \toprule 
	\addlinespace
    General    & \multicolumn{1}{C{0.12}}{Default} & \multicolumn{1}{C{0.12}}{Ant} & \multicolumn{1}{C{0.12}}{Humanoid} \\                             
    \cmidrule(lr){1-1}     \cmidrule(lr){2-2} \cmidrule(lr){3-3} \cmidrule(lr){4-4} 
    Discount rate ($\gamma$)      			& $0.995$ \hspace*{0.5em} \\    
	GAE parameter ($\lambda$)      			& $0.97$  \hspace*{0.5em} \\
	Minibatches per epoch 					& $32$    \hspace*{0.5em} \\
	Epochs per update 						& $10$    \hspace*{0.5em} \\
	Value function optimizer 				& Adam    \hspace*{0.5em} \\
	Value function learning rate 			& $3\mathrm{e}{-4}$  \hspace*{0.5em} \\
	Policy optimizer						& Adam    \hspace*{0.5em} \\
	Initial policy learning rate ($\eta$) 	& $3\mathrm{e}{-4}$ \hspace*{0.5em} & $1\mathrm{e}{-4}$ \hspace*{0.5em} & $3\mathrm{e}{-5}$ \hspace*{0.5em}  \\
	Initial policy std. deviation multiple 	& $1.0$ \hspace*{0.5em} & $0.5$ \hspace*{0.5em} & $0.5$ \hspace*{0.5em}  \\
	\\
	PPO \\
    \cmidrule(lr){1-1} 	
    Clipping parameter ($\epsilon^{\textnormal{PPO}}$) 	& $0.2$ \hspace*{0.5em} \\
    Batch size ($N$) 								& $2{,}048$ \hspace*{0.5em} \\
	\\
	GePPO \\
    \cmidrule(lr){1-1} 	
    Clipping parameter ($\epsilon^{\textnormal{GePPO}}$) $^*$ 	& $0.1$ \hspace*{0.5em} \\
    Number of prior policies ($M$) $^*$ 	                & $4$ \hspace*{0.5em} \\    
	Minimum batch size ($n$) & $1{,}024$  \hspace*{0.5em} \\
	Adaptive factor ($\alpha$) & $0.03$ \hspace*{0.5em} \\
	Minimum threshold factor ($\beta$) & $0.5$  \hspace*{0.5em} \\
	V-trace truncation parameter ($\bar{c}$) & $1.0$ \hspace*{0.5em} \\
	\\
    \cmidrule(lr){1-1}
	\footnotesize{$^*$ Represents calculated value.} \\
	\addlinespace
    \bottomrule
  \end{tabular}
\end{table}

To aid in reproducibility, we describe the implementation details used to produce our experimental results. Note that all choices are based on the default implementation of PPO in \citet{henderson_2018}.

\paragraph{Network structures and hyperparameters}

As discussed in Section~\ref{sec:experiments}, we represent the policy $\pi$ as a multivariate Gaussian distribution where the mean action for a given state is parameterized by a neural network with two hidden layers of 64 units each and tanh activations. The state-independent standard deviation is parameterized separately, where for each action dimension the standard deviation is initialized as a multiple of half of the feasible action range. The value function $V^{\pi}(s)$ is parameterized by a separate neural network with two hidden layers of 64 units each and tanh activations. Observations are standardized using a running mean and standard deviation throughout the training process. Both the policy and value function are updated at every iteration using minibatch stochastic gradient descent. All hyperparameters associated with these optimization processes can be found in Table~\ref{tab:hyperparameters}. Due to the high-dimensional nature of Ant-v3 and Humanoid-v3, we tuned the initial learning rate and standard deviation multiple of the policy used for PPO on these tasks. These values can also be found in Table~\ref{tab:hyperparameters}. For a fair comparison, we used the same hyperparameter values for GePPO.

Following \citet{henderson_2018}, we consider the clipping parameter $\epsilon^{\textnormal{PPO}} = 0.2$ and a batch size of $N=2{,}048$ for PPO. Because sample trajectories for the tasks we consider can contain up to one thousand steps, we consider a minimum batch size of $n=1{,}024$ for GePPO. When writing the default batch size in PPO as $N=Bn$, this results in $B=2$. Using this value of $B$, the number of prior policies $M$ and the corresponding policy weights $\nu$ for GePPO are calculated according to Theorem~\ref{thm:geppo_essopt}. These weights are shown on the right-hand side of Figure~\ref{fig:polweights}. The clipping parameter $\epsilon^{\textnormal{GePPO}}$ is calculated based on these policy weights using Lemma~\ref{lem:geppo_param}, which results in $\epsilon^{\textnormal{GePPO}}=0.1$. Finally, the adaptive factor $\alpha$ and minimum threshold factor $\beta$ used for our adaptive learning rate method can be found in Table~\ref{tab:hyperparameters}. 

\paragraph{Advantage estimation}

Advantages $A^{\pi_k}(s,a)$ of the current policy in PPO are estimated using Generalized Advantage Estimation (GAE) \citep{schulman_2016} with $\lambda = 0.97$. Note that the parameter $\lambda$ in GAE determines a weighted average over $K$-step advantage estimates. When samples are collected using the current policy $\pi_k$, these multi-step advantage estimates are unbiased except for the use of bootstrapping with the learned value function. In GePPO, however, we must estimate $A^{\pi_k}(s,a)$ using samples collected from prior policies, so the multi-step advantage estimates used in GAE are no longer unbiased. Instead, we use V-trace \citep{espeholt_2018} to calculate corrected estimates that are suitable for the off-policy setting. V-trace corrects multi-step estimates while controlling variance by using truncated importance sampling. For a learned value function $V$, current policy $\pi_k$, and prior policy $\pi_{k-i}$ used to generate the data, this results in the following $K$-step target for the value function:
\begin{equation}
V^{\pi_k}_{\textnormal{trace}}(s_t) = V(s_t) + \sum_{j=0}^{K-1} \gamma^j \left( \prod_{i=0}^j c_{t+i} \right) \delta_{t+j}^V,
\end{equation}
where $\delta_{t}^V = r(s_{t},a_{t}) + \gamma V(s_{t+1}) - V(s_{t})$ and $c_{t} = \min \left( \bar{c}, \nicefrac{\pi_k (a_{t} \mid s_{t} )}{\pi_{k-i} (a_{t} \mid s_{t} )  } \right)$ represents a truncated importance sampling ratio with truncation parameter $\bar{c}$. We can use the same correction techniques to generate $K$-step advantage estimates from off-policy data. For $K \geq 2$, the corresponding $K$-step advantage estimate is given by
\begin{equation}
A^{\pi_k}_{\textnormal{trace}}(s_t,a_t) = \delta_{t}^V + \sum_{j=1}^{K-1} \gamma^j \left( \prod_{i=1}^j c_{t+i} \right) \delta_{t+j}^V,
\end{equation}
and for $K=1$ we have the standard one-step estimate $A^{\pi_k}_{\textnormal{trace}}(s_t,a_t) = \delta_{t}^V$ that does not require any correction. We use $\bar{c}=1.0$ in our experiments, which is the default setting in \citet{espeholt_2018}. Note that \citet{espeholt_2018} treat the final importance sampling ratio in each term separately, but in practice the truncation parameters are chosen to be the same so we do not make this distinction in our notation. Finally, we consider a weighted average over these corrected multi-step advantage estimates as in GAE.

Typically, the resulting advantage estimates are standardized within each minibatch of PPO. Note that the expectation of $A^{\pi_k}(s,a)$ with respect to samples generated by $\pi_k$ is zero, so the centering of advantage estimates ensures that the empirical average also satisfies this property. In the off-policy setting, the appropriate quantity to standardize is the starting point of policy updates 
\begin{equation}\label{eq:geppo_start}
\frac{\pi_k(a \mid s)}{\pi_{k-i}(a \mid s)} A^{\pi_k}(s,a),
\end{equation}
since the expectation of \eqref{eq:geppo_start} with respect to samples generated by prior policies is zero. Note that the standardization of \eqref{eq:geppo_start} generalizes the standardization done in PPO, where the probability ratios at the beginning of each policy update are all equal to one.

\paragraph{Computational resources}

All experiments were run on a Windows 10 operating system, Intel Core i7-9700 CPU with base speed of $3.0$ GHz and 32 GB of RAM, and NVIDIA GeForce RTX 2060 GPU with 6 GB of dedicated memory. OpenAI Gym \citep{brockman_2016} is available under The MIT License, and we make use of MuJoCo \citep{todorov_2012} with a license obtained from Roboti LLC. Using code that has not been optimized for execution speed, simulations for all seeds on a given environment required approximately 3 hours of wall-clock time for PPO and 4 hours of wall-clock time for GePPO. The increased wall-clock time of GePPO is due to the fact that GePPO performs twice as many updates in our experiments compared to PPO. Note that on average the sample efficiency benefits of GePPO not only lead to improved performance compared to PPO for a fixed number of samples, but also for a fixed amount of wall-clock time. 

%%%%%%%%%%%%%%%%%%%%%%%%%%%%%%%%%%%%%%%%%%%%%%%%%%%%%%%%%%%%

\section{Additional experimental results}\label{sec:app_exp}

\begin{figure}[t]
\centering
\includegraphics[width=1.00\linewidth]{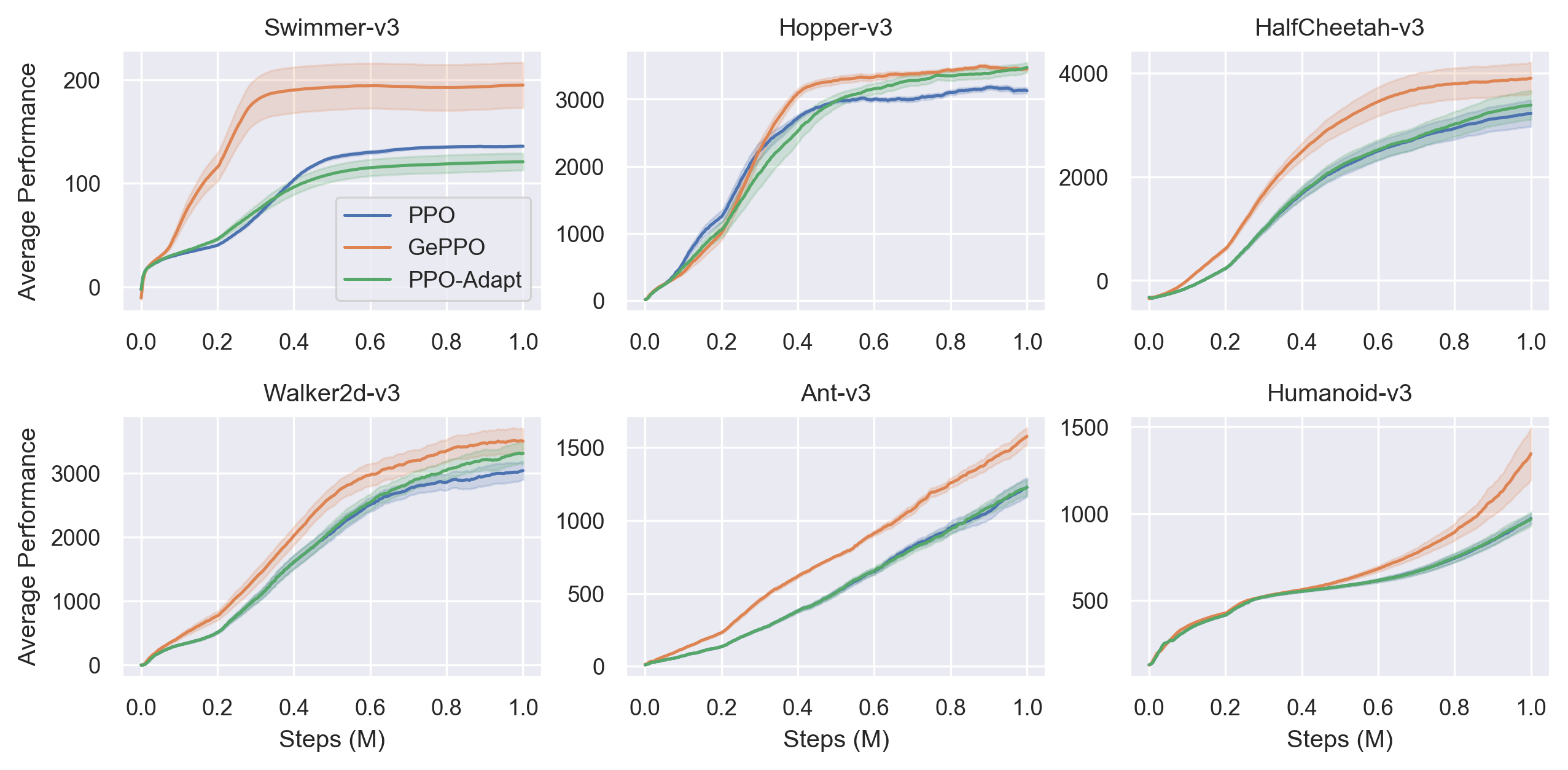}
\caption{Performance throughout training across MuJoCo tasks, where the policy network has two hidden layers of 64 units each. Shading denotes half of one standard error. PPO-Adapt represents PPO with the addition of our adaptive learning rate method.}\label{fig:J_ablation_6464}
\end{figure}

\begin{figure}
\centering
\includegraphics[width=1.00\linewidth]{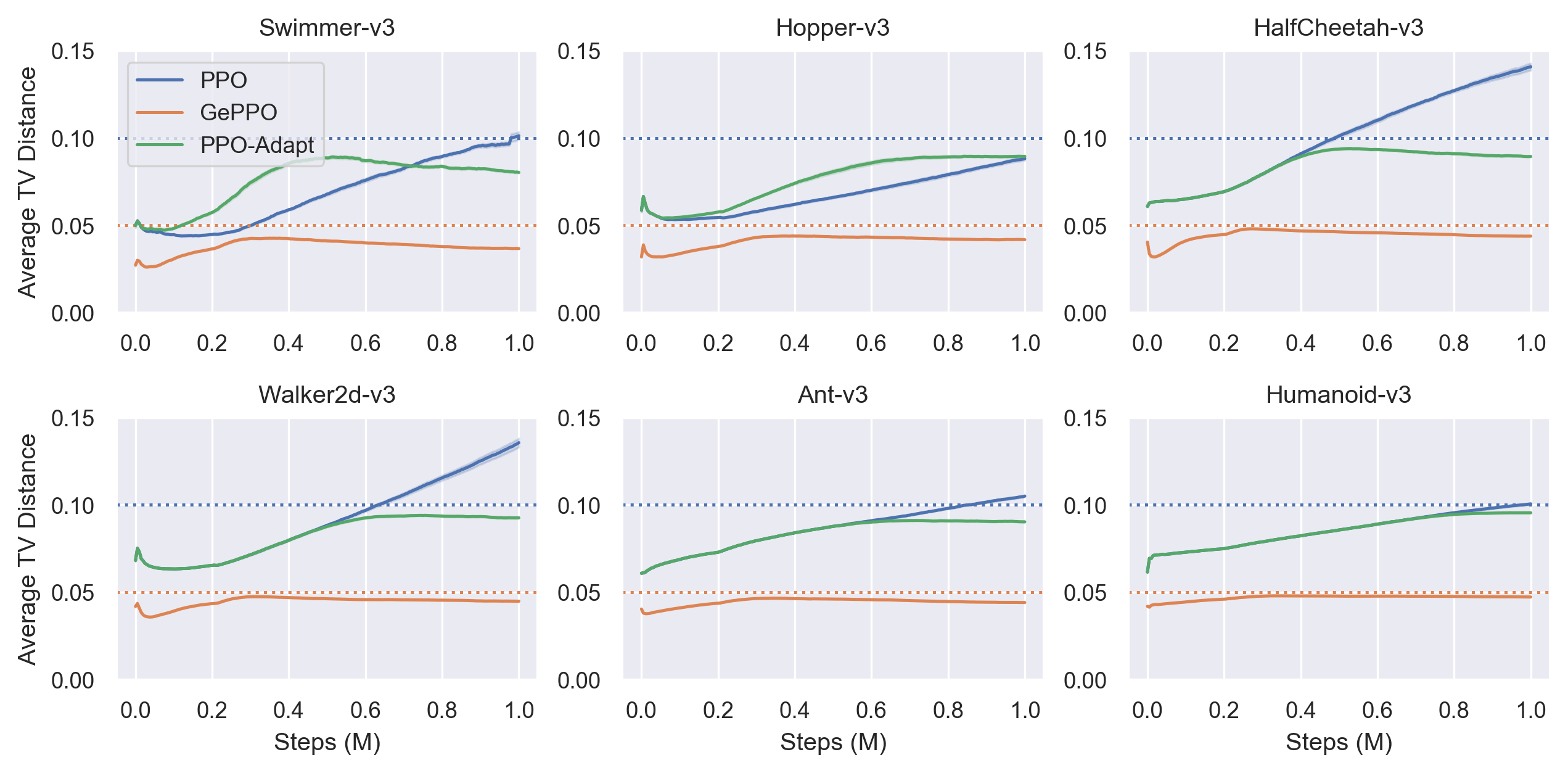}
\caption{Change in average total variation distance per policy update across MuJoCo tasks, where the policy network has two hidden layers of 64 units each. Shading denotes half of one standard error. PPO-Adapt represents PPO with the addition of our adaptive learning rate method. Horizontal dotted lines represent target total variation distances for PPO and GePPO, respectively.}\label{fig:tv_ablation_6464}
\end{figure}

In this section, we provide additional experimental results to further analyze the performance of our algorithm. In particular, we include results across all MuJoCo tasks for both the default policy network with two hidden layers of 64 units each and a standard wide policy network with two hidden layers of 400 and 300 units, respectively. In all cases, we also consider a variant of PPO that uses our adaptive learning rate method to better understand the drivers of performance in GePPO.

Figure~\ref{fig:J_ablation_6464} and Figure~\ref{fig:tv_ablation_6464} show the performance throughout training and change in average total variation distance per policy update, respectively, when using the default policy network. Note that the results shown in Figure~\ref{fig:Jplot} and Figure~\ref{fig:wide} for PPO and GePPO with the default policy network are repeated here for reference. As discussed in Section~\ref{sec:experiments}, we see that the change in total variation distance per policy update in PPO increases throughout training across all environments. Because hyperparameters were tuned using the default policy network, this trend does not lead to unstable performance. The addition of our adaptive learning rate to PPO results in comparable performance in this well-tuned setting, while ensuring that the change in total variation distance remains close to the target determined by the clipping parameter. This causes a decrease in total variation distance change per policy update compared to PPO at the end of training in most environments, and an increase compared to PPO at the beginning of training in Swimmer-v3 and Hopper-v3. Finally, note that GePPO outperforms the variant of PPO with our adaptive learning rate, which indicates that sample reuse is an important driver of performance in GePPO. On average, GePPO improves average performance throughout training by 31\% compared to PPO and 32\% compared to PPO with our adaptive learning rate.

\begin{figure}
\centering
\includegraphics[width=1.00\linewidth]{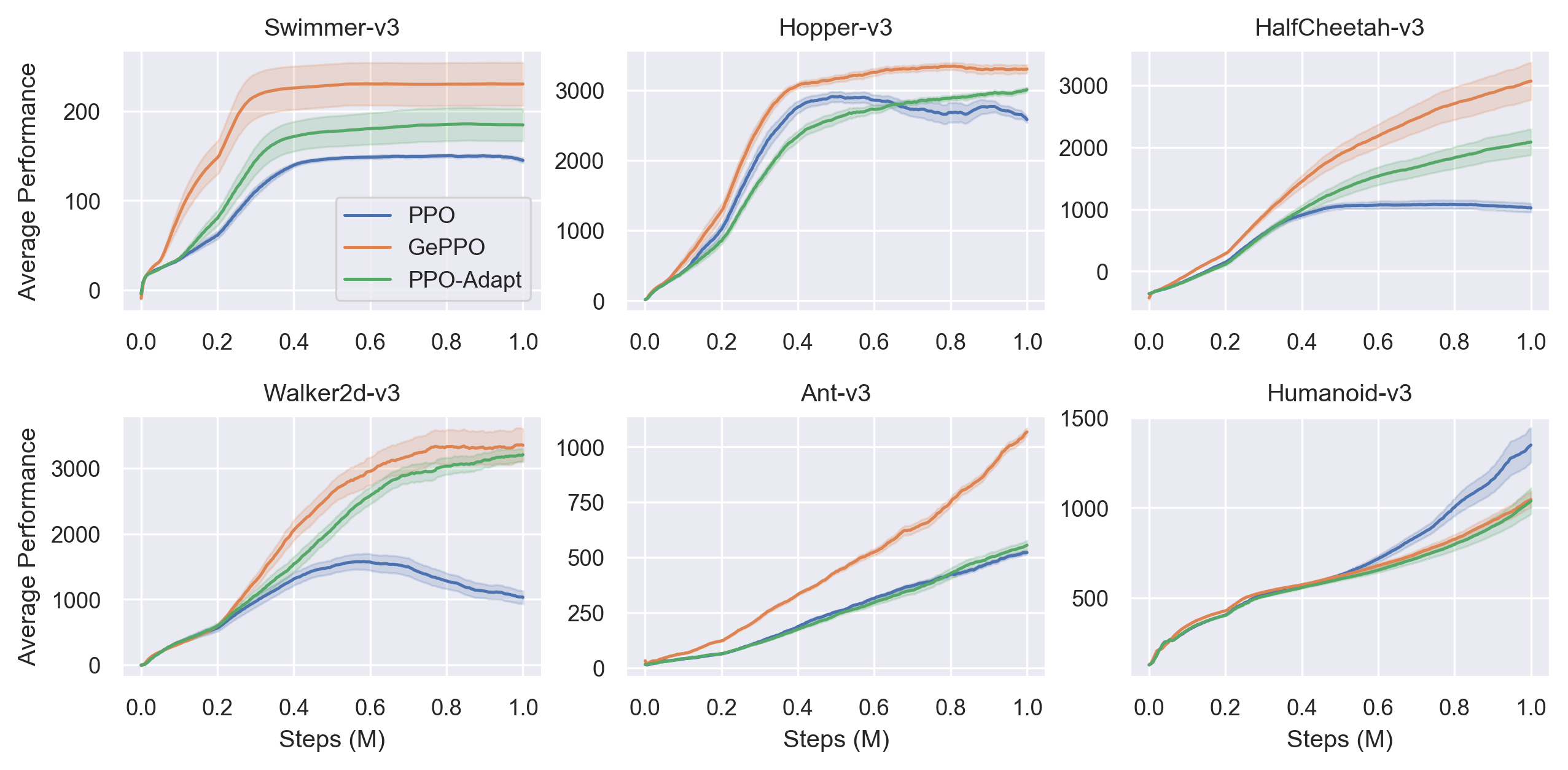}
\caption{Performance throughout training across MuJoCo tasks, where the policy network has two hidden layers of 400 and 300 units, respectively. Shading denotes half of one standard error. PPO-Adapt represents PPO with the addition of our adaptive learning rate method.}\label{fig:J_ablation_400300}
\end{figure}

\begin{figure}
\centering
\includegraphics[width=1.00\linewidth]{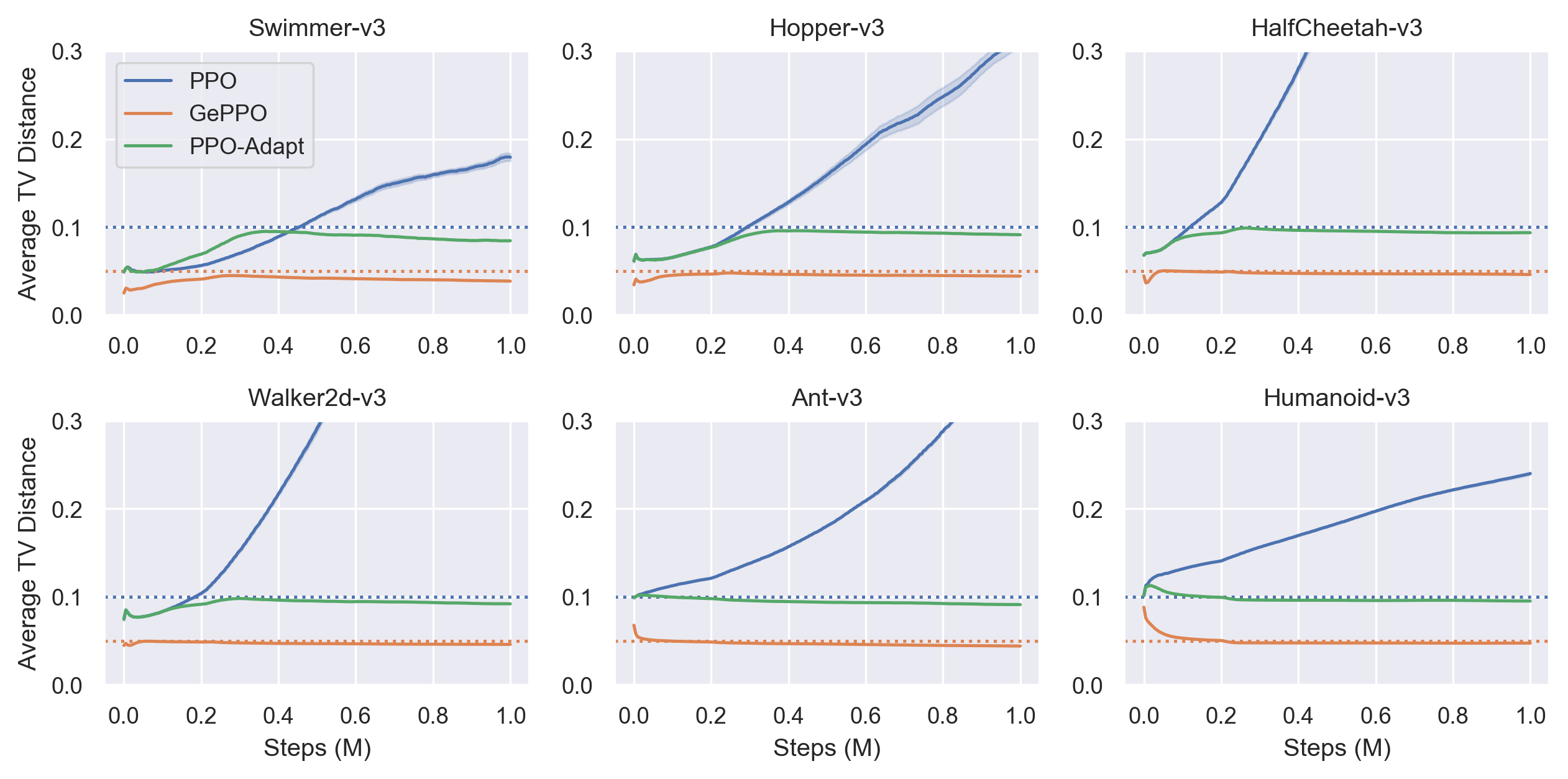}
\caption{Change in average total variation distance per policy update across MuJoCo tasks, where the policy network has two hidden layers of 400 and 300 units, respectively. Shading denotes half of one standard error. PPO-Adapt represents PPO with the addition of our adaptive learning rate method. Horizontal dotted lines represent target total variation distances for PPO and GePPO, respectively.}\label{fig:tv_ablation_400300}
\end{figure}

Figure~\ref{fig:J_ablation_400300} and Figure~\ref{fig:tv_ablation_400300} show the same metrics for a standard wide policy network with two hidden layers of 400 and 300 units, respectively. The results shown in Figure~\ref{fig:wide} for Walker2d-v3 are repeated here for reference. As described in Section~\ref{sec:experiments} for Walker2d-v3, we see that the wide policy network exacerbates the total variation distance trend observed in PPO under the default settings. The increased size of the policy network causes the probability ratio to be more sensitive to gradient updates during training, which renders the clipping mechanism ineffective and leads to excessively large policy updates. These large updates result in poor performance across several environments, and even cause performance to decline over time for Hopper-v3 and Walker2d-v3. The addition of our adaptive learning rate to PPO controls this source of instability, resulting in improved performance throughout training and stable policy updates with acceptable levels of risk. The exception to this trend is Humanoid-v3, where PPO outperforms PPO with our adaptive learning rate. This suggests that a more aggressive risk profile may improve performance on Humanoid-v3 without sacrificing training stability. Finally, by reusing samples from prior policies, GePPO further improves upon the performance of PPO with our adaptive learning rate. On average, GePPO improves average performance throughout training by 63\% compared to PPO and 34\% compared to PPO with our adaptive learning rate.

From these results, we see that the principled sample reuse in GePPO is a key driver of performance gains compared to PPO, while the use of an adaptive learning rate is important for stable performance that is robust to hyperparameter settings. Note that it is possible to achieve strong performance without the use of our adaptive learning rate, but this requires careful tuning of the learning rate that depends on several factors including the environment, policy network structure, length of training, and other hyperparameter settings.

%%%%%%%%%%%%%%%%%%%%%%%%%%%%%%%%%%%%%%%%%%%%%%%%%%%%%%%%%%%%

\end{document}